\documentclass[twoside,11pt]{article}

\usepackage{natbib}
\usepackage{url}
\usepackage{fullpage}
\usepackage{amsthm}
\usepackage{amsmath}
\usepackage{pifont}
\usepackage{bbold}
\usepackage{float}
\usepackage{lipsum}
\usepackage{lineno}
\usepackage{bbm}
\usepackage{tikz}
\usetikzlibrary{shapes.geometric,positioning}
\usepackage[ruled,vlined,linesnumbered,noresetcount]{algorithm2e}
\usepackage{setspace}
\usepackage{multirow}
\usepackage{subcaption}
\usepackage{array,ulem}
\usepackage{enumerate}

\tikzset{
	semi/.style={
		semicircle,
		draw,
		minimum size=2em
	}
}
\newcolumntype{M}[1]{>{\centering\arraybackslash}m{#1}}
\newcommand{\pa}{\mathrm{pa}} 
\newcommand{\aps}{\mathrm{aps}} 
\newcommand{\an}{\mathrm{an}} 
\newcommand{\de}{\mathrm{de}} 
\def\ci{\perp\!\!\!\perp}

\newcommand{\pr}{\mathbbm{P}}
\newcommand{\ex}{\mathbbm{E}}

\newtheorem{theorem}{Theorem}
\newtheorem{definition}{Definition}
\newtheorem{proposition}{Proposition}

\begin{document}

\title{Fair Data Adaptation with Quantile Preservation}

\author{ Drago Ple{\v c}ko and Nicolai Meinshausen \\  Seminar f{\"u}r Statistik\\
       ETH Z{\"u}rich\\
       Z{\"u}rich, 8092, Switzerland}


\maketitle

\begin{abstract}
Fairness of classification and regression has received much attention recently and various, partially non-compatible, criteria have been proposed. The fairness criteria can be enforced for a given classifier or, alternatively, the data can be adapated to ensure that every classifier trained on the data will adhere to desired fairness criteria. We present a practical data adaption method based on quantile preservation in causal structural equation models. The data adaptation is based on a presumed counterfactual model for the data. While the counterfactual model itself cannot be verified experimentally, we show that certain population notions of fairness are still guaranteed even if the counterfactual model is misspecified. The precise nature of the fulfilled non-causal fairness notion (such as demographic parity, separation or sufficiency) depends on the structure of the underlying causal model and the choice of  resolving variables. We describe an implementation of the proposed data adaptation procedure based on Random Forests \citep{breiman2001random} and demonstrate its practical use on simulated and real-world data.
\end{abstract}

\section{Introduction}
 Care needs to be taken when machine learning techniques are used in socially sensitive domains, because algorithms are sometimes capable of learning societal biases we would not want them to learn. For example, women tend to be disadvantaged in  credit score ratings, partially due to the fact that women are currently perceived to have lower income on average \citep{blau2003}. A gender-neutral credit scoring would be desirable.  The precise notion of fairness one would like to achieve is  often debatable, though.  A much publicised and discussed example of this is the COMPAS dataset \citep{ProPublica} which involves predicting whether inmates will recidivate after being released on parole. The prediction is based on demographic data and information about prior convictions. Standard methods, which do not ensure all racial groups are treated in the same way, have been shown to lead to highly-discriminatory predictions against the black population \citep{ProPublica}. After being criticised for not achieving a fairness criterion called equality of odds, the company that produced the predictions, Northpoint, has claimed that their predictions satisfy a criterion called calibration \citep{dieterich2016compas}.

Current approaches for building fair predictors broadly fall into three categories. Pre-processing methods focus on transforming the data in order to remove any unwanted bias \citep{zemel2013,calmon2017}. In-processing methods attempt to build in fairness constraints into the training step \citep{fish2016,zafar2017,empiricalrisk}. Post-processing methods focus on transforming an already constructed predictor \citep{hardt2016}. Our work falls into the pre-processing category. The work that relates to our approach are the path-specific counterfactuals methods, proposed by \cite{shpitser2018} and \cite{pathspecific}, which aim to detect and eliminate the path specific effects (PSE) of the protect attribute on the response.

In particular, our work
\begin{enumerate}
\item Provides a practical implementation of fair data adaption based on Random Forests and an underlying causal model which is assumed to be known. This allows to incorporate resolving variables \citep{kilbertus2017}. The software is provided as an R package \texttt{fairadapt}.
\item Presumes a specific counterfactual model. We can show that a counterfactual notion of fairness is satisfied if the model is correct (unfortunately not verifiable), but that certain population fairness notions are satisfied in any case, even if the counterfactual model is wrong.
\item Allows the achieved notion of fairness to depend on the causal graph. This might offer a more principled way of agreeing on a suitable fairness notion (if people can agree on the structure of the underlying causal graph).
\end{enumerate}
We also demonstrate the empirical value of our approach, exhibiting very competitive performance.

\subsection{Setup}
Let random variable $Y\in \mathcal{Y}$ be the outcome of interest that one would like to predict in the future in a fair way.  For simplicity, we mostly assume binary classification so that  $\mathcal{Y} = \lbrace 0, 1 \rbrace$. The binary outcome $Y$ represents perhaps recidivism whilst on parole or repayment of a loan. Let $A$ be the protected attribute such as race or gender and $X=(X^{(1)}, \ldots, X^{(p)}) \in \mathbbm{R}^p$ be predictor variables for the outcome of interest. We assume we have access to $n$  i.i.d. samples $(A_i,{X}_i,Y_i)$ $i=1.\ldots,n$ coming from a distribution $F_{A,{X},Y}$. For the majority of the exposition we assume that $A$ has two levels $\lbrace 0,1 \rbrace$, but generalizations are straightforward. The key goal is to provide a data-transformation or data-adaptation
\[T : \mathbbm{R}^p\times \mathcal{Y} \mapsto \mathbbm{R}^p\times \mathcal{Y}. \] The transformation should be such that if we train a classifier with the adapted data
\[  T\big((X_i,Y_i)\big), \ i=1,\ldots,n \]
instead of the original data $\lbrace (X_i,Y_i),\ i=1,\ldots,n \rbrace$, we want to be able to automatically guarantee appropriate fairness criteria. At the same time, we want the change induced by the data adaptation to be minimal in an appropriate sense.

\subsection{Causal framework} \label{causal}
We mainly use a standard non-parametric structural equation model (NPSEM)  for $Z=(A,X,Y)\in \mathbb{R}^{p+2}$ as in \citet{pearl2000} and let each variable $Z^{(k)}$ be defined as
\begin{equation}\label{eq:scm} Z^{(k)} = g_k( Z^{\rm{pa}_k}, U^{(k)}),\end{equation}
where $U \in \mathbbm{R}^{p+2}$ is a latent variable that determines the realization of the variable $Z^{(k)}$ and $\pa_k$ is the set of parents of the variable $Z^{(k)}$. We denote by $f_k(z^{(k)} \mid z^{\pa_k})$ the density corresponding to $Z^{(k)}$. Without limitation of generality, we let $U$ have marginally a uniform distribution on $[0,1]$. We further assume that $g_k(z,u)$ is monotonically increasing with $u$ for each $z$. Thus, $U^{(k)}$ can be interpreted as the quantile of the $k$-th variable, conditional on the value of its parents $Z^{\rm{pa}_k}$. We also assume that the components of $U$ are independent, that is we assume lack of confounding, also known as the Markovian assumption in \citep{pearl2009} . Assume for the following discussion first that the random variables are continuous and a density exists. Then $Z(U=u)=z$ is the realized value of $Z$ under the realization $u$ of the quantiles $U$ and there exists a one-to-one mapping between the value of $z$ and $u$. We will return later to the case of discrete random variables and randomization, where the deterministic relationship between $z$ and $u$ breaks down.

We will try to keep notation as lean as possible. Suppose $R$ is a subset of $X$. We denote with
$X(A=a,R=r)$ the random variable that is determined by the set of structural equations~\eqref{eq:scm} where
the structural equation for variables $A$ and $R$ are replaced by
setting their corresponding values equal to $a$ and $r$ respectively. We can, for example, compare  the distribution of the predictor variable $X$ under the interventions $\text{do}(A=a)$ and $\text{do}(A=a')$ by comparing the distributions of
\[X(A=a) \text{ and } X(A=a')\]
without making assumptions about the joint distribution of these two random variables and could use
the formal framework of single-world intervention graphs (SWIGs) introduced  in \citet{swigs}.

However,  we need cross-world statements for some of the discussion.
The value $ X(A=a,U=u)$ is the outcome under a specific realization of the latent variable $U$.
We can view the value $X(A=a,U=u)$ as the realized value of $X$ under the quantiles specified by $u$ and an intervention $\text{do}(A=a)$. In this sense, a simple counterfactual model defines $X(A=a,U=u)$ to be the counterfactual of $X(U=u)$ under the do-intervention that is setting the attribute $A$ to the value $a$. This way we are defining a joint distribution for $X$ and $X(A=a)$ and can make cross-world statements. Of course, we could define a different counterfactual model by keeping the marginal distributions of $X$ and $X(A=a)$ identical, but changing the joint distribution. For the specific counterfactual model used in this paper we keep the conditional quantiles $U$ constant.
\begin{definition}[Quantile preservation assumption (QPA)] \label{def:qpa}
	The conditional quantiles $U$ of the SEM in Equation \eqref{eq:scm} remain unchanged under a $do(A = a)$ intervention.
\end{definition}
\noindent Such a counterfactual cross-world model (assumptions about the joint distribution of random variables under different interventions) is obviously not empirically verifiable. We will try to make clear where we use single-world and where we use cross-world assumptions throughout the text.

Another assumption we make is that the protected attribute $A$ is a root of the causal graph $\mathcal{G}$. A consequence of this is that the $do(A = a)$ intervention is equivalent to conditioning on $A = a$. This idea, which is easily shown using the 2nd rule of do-calculus \citep{pearl2009}, shows up frequently in our discussion.

\subsection{Example}
Consider the following example with four variables $A,\ X_1,\ X_2,\ Y$. Variable $A$ is the protected attribute, in this case gender ($A = 0$ corresponding to male, $A = 1$ to female). Let $X_1$ be the education level and $X_2$ the current salary of an individual. Outcome $Y$ is the successful repayment of a loan. Edges in the graph indicate how variables affect each other.
	\begin{figure}[H] \centering
			\begin{tikzpicture}
			[>=stealth, rv/.style={circle, draw, thick, minimum size=6mm}, rvc/.style={triangle, draw, thick, minimum size=10mm}, node distance=18mm]
			\pgfsetarrows{latex-latex};
			\begin{scope}
			\node[rv] (1) at (-2,0) {$A$};
			\node[rv] (2) at (0,1) {$X_1$};
			\node[rv] (3) at (0,-1) {$X_2$};
			\node[rv] (4) at (2,0) {$Y$};
			\draw[->] (1) -- (2);
			\draw[->] (1) -- (3);
			\draw[->] (2) -- (3);
			\draw[->] (2) -- (3);
			\draw[->] (3) -- (4);
			\draw[->] (2) -- (4);
			\end{scope}
			\end{tikzpicture}
		\end{figure}
 \noindent The main problem is that the attribute $A$, gender, has an effect on both $X_1$ and $X_2$ (education and salary). We want to find a data transformation that makes the data ``look" the same for all levels of $A$. Subject to this, we also want to minimize the distortion in the data coming from the transformation. Namely, for all females ($A = 1$), we would first want to compute their education level  had they been male. More explicitly, for a female with education level $ x_1$, we give it the transformed value $ \widetilde{x}_1$ chosen such that $$\pr(X_1 \geq x_1 \mid A = 1) = \pr(X_1 \geq \widetilde{x}_1 \mid A = 0). $$
The main idea is that the \textit{relative education within the subgroup} would stay the same if we changed someone's gender. If you are a female better than 60\% of the females in the dataset, we assume you would be better than 60\% of males had you been male. After computing everyone's education (in the ``male" world), we continue by computing the transformed salary values $\widetilde{X}_2$. The approach is again similar, but this time we condition on the education level as well. That is, a female with values $(X_1, X_2) = (x_1, x_2)$ is assigned a salary level $ \widetilde{x}_2$ such that
$$\pr(X_2 \geq x_2 \mid X_1 = x_1,\ A = 1) = \pr(X_2 \geq \widetilde{x}_2 \mid X_1 = \widetilde{x}_1,\ A = 0),$$
where the value $\widetilde{x}_1$ was obtained in the previous step. The transformed data $\widetilde{X}_1, \widetilde{X}_2$ can then be used to construct a classifier. Generally, the transformation we are describing is carried out using \textit{quantile regression forests} \citep{qrf}. A full implementation of this method for a general situation is available in the \texttt{fairadapt} package on CRAN. The aim of this paper is to formalize all of the ideas above mathematically.

As we transform our data, we also end up with the covariate values individuals would attain if we set their gender to male. We are basically trying to answer the hypothetical question ``What would my attributes look like, had I been male?". This aspect of our method can help explain \textit{why} individuals obtain their predictions. To this we refer to as method \textit{interpretability}.
\subsection{Related work}
We summarize the work which is relevant to our discussion. Our method is related to fairness notions of \textit{demographic parity} \citep{darlington1971} and \textit{calibration} \citep{chouldechova2017}. In a simple case, we can also see a connection to \textit{equality of odds} \citep{hardt2016}. A discussion of fair adaptation and observational criteria is given in Section \ref{relationobservational}. Two causal notions of fairness that play an important part in our method are \textit{counterfactual fairness} \citep{counterfactualfairness} and \textit{resolving variables} \citep{kilbertus2017}. More about relation to causal notions of fairness is discussed in Section \ref{causalnotions}. Methods based on mediation analysis build on some related ideas \citep{zhang2018eo, zhang2018fairness}. In their approach, the two most similar methods to ours are the path specific counterfactual methods \citep{shpitser2018, pathspecific}. This relation is discussed more in Section \ref{pathspecmethods}.

\subsection{Structure of the paper}
In Section \ref{causalnotions} causal notions of fairness that are important for our method are discussed. In Section \ref{adaptation} the adaptation procedure is introduced and a summary of the goals when applying our method is provided. The achieved fairness notions are discussed and in particular how they depend (or do not depend) on the assumptions used. The population level adaptation procedure (under no estimation error) is given. The assumptions that are used are also briefly discussed. We illustrate what our desired fairness notions amount to in the simplest linear additive setting. In Section \ref{relation} the relation of our work to previously proposed methods and criteria is analyzed. Section \ref{methodformalisation} goes in depth about discussing the practical aspects of our method. How to handle discrete variables in the adaptation procedure is particularly emphasized. The sample level non-parametric adaptation procedure is given. The training step options after applying the data adaptation are described. Two possible methodological extensions are discussed in the end of the section. In Section \ref{experimental} empirical performance of our method is demonstrated both on simulated and real-world data.
\section{Causal notions of fairness} \label{causalnotions}
We look at two counterfactual notions of fairness that play an important role in our methodology.
\subsection{Counterfactual fairness}
	 \label{firstCF} \textit{Counterfactual fairness} was first introduced as a notion by \citet{counterfactualfairness}. The notation $\widehat{Y}(A=a)$ indicates again the prediction under a do-intervention $\text{do}(A=a)$ on the protected attribute.
\begin{definition}[Counterfactual fairness, \citet{counterfactualfairness}] \label{CFF}
		A predictor $\widehat{Y}$ is counterfactually fair if
		\begin{equation} \label{CFFcrit2}
		\widehat{Y}(A = a) \mid A=a, X=x \quad \stackrel{d}{=} \quad  \widehat{Y}(A = a') \mid A = a, X=x \;\;\forall a,a',x
		\end{equation}
		Here $\widehat{Y}(A=0)$ indicates that $\widehat{Y}$ comes from the distribution resulting from a $do(A = 0)$ intervention, while the conditioning occurs in the observational, non-interventional, distribution.
	\end{definition}
	The idea behind this notion is that if we intervene to change someone's race or gender, this should not affect the prediction they obtain. We emphasize that the notion in Equation \eqref{CFFcrit2} in our setting is a single-world counterfactual notion. For the original authors this notion is a cross-world one (due to the existence of latent variables $U$ which remain distributions even after the conditioning on $A = a, X = x$).

\noindent A weaker form of counterfactual fairness would just require that the distribution of $\widehat{Y}$ under an intervention on the protected attribute remains unchanged, that is
\begin{definition}[Population fairness] \label{def:popfair}
A predictor $\widehat{Y}$ is said to satisfy population fairness if
$$\widehat{Y}(A = a) \overset{d}{=} \widehat{Y} \quad \forall a.$$
\end{definition}
\noindent The distributional equivalence from Definition \ref{def:popfair} does not rest on cross-world assumptions and is equal to the observational criterion of demographic parity, in the case when $A$ is a root node in the causal graph (shown later in Proposition \ref{prop:demparity}).
In contrast, a much stronger notion can be defined as follows
\begin{definition}[Strong counterfactual fairness] \label{def:strongfair}
A predictor $\widehat{Y}$ is said to satisfy strong counterfactual fairness if $$\widehat{Y}(A=a, U = u) = \widehat{Y}(U = u)\;\;\forall a,u \ .$$
\end{definition}
\noindent Definition \ref{def:strongfair} requires the counterfactual prediction to be identical when setting the protected attribute to any value. Note that this is an individual level fairness notion.

\subsection{Resolving variables}\label{secondCF}
\citet{kilbertus2017} discuss that in some situations the protected attribute $A$ can affect variables in a non-discriminatory way. For instance, in the Berkeley admissions dataset \citep{bickel1975sex} we observe that females often apply for departments with lower admission rates and consequently have a lower admission probability. However, we perhaps would not wish to account for this difference in the adaptation procedure if we were to argue that department choice is a choice everybody is free to make. This motivated the following definition:
	\begin{definition}[Resolving variables, \citet{kilbertus2017}] \label{resolving}
		Let $\mathcal{G}$ be the causal graph of the data generating mechanism. Let the descendants of variable $A$ be denoted by $\de(A)$. A variable $R$ is called resolving if
		\begin{enumerate}[(i)]
			\item $R \in \de(A)$
			\item the causal effect of $A$ on $R$ is considered to be non-discriminatory
\end{enumerate}
\end{definition}

\noindent The idea is that the value of a resolving variable, or a resolver, $R$ should not change under our adaptation procedure. More generally, we can consider a set of resolving variables ${R}$. The desired counterfactual fairness criteria with respect to this definition can now be stated as
\begin{align} \label{RCFF1}
&\text{Population: }\widehat{Y}(A = a, {R} = {r})  \quad \stackrel{d}{=}\quad  \widehat{Y}(A = a', {R} = {r})\;\; \forall r \\ \label{RCFF2}
&\text{Cond.: }\widehat{Y}(A = a, {R} = {r}) \mid A = a, {X} = {x} \quad \stackrel{d}{=}\quad  \widehat{Y}(A = a', {R} = {r}) \mid A = a, {X} = {x} \;\;\forall a,x,r \\ \label{RCFF3}
&\text{Strong: }\widehat{Y}(A = a, {R} = {r}, U = u)  \quad = \quad  \widehat{Y}(A = a', {R} = {r}, U = u)  \;\;\forall a,a',r,u
\end{align}
to which we refer to as population, conditional and strong resolved fairness respectively. In the presence of resolving variables, we have an additional $do({R}={r})$ intervention, while the three different levels of fairness stay the same. The strong notion requires that the counterfactual predictions remain unchanged under a do-intervention on the protected attribute.

It is not immediately clear which variables should be considered as resolving. It can even happen that the same variable can be resolving or non-resolving in different applications. For instance, when recruiting students for an athletics training programme, we perhaps do not wish to give males an advantage based on physical ability. In this case, physical strength is not a resolving variable. However, if we are hiring workers for a physical job, we might want to consider physical strength as a resolving variable.

\section{Adaptation} \label{adaptation}
The main goal of this paper is to combine the two causal notions given in Sections \ref{firstCF} and \ref{secondCF} to describe a preprocessing procedure which gives a fair representation of the data. After this, any method can be used to construct a fair classifier $\widehat{Y}$, with slight care in the training step.

\paragraph{Adaptation aim.} Suppose $R$ is the set of resolving variables. We want to find a transformation $FT: \mathbbm{R}^p \mapsto \mathbbm{R}^p$ such that the transformed data $FT(X)$ satisfy
\begin{equation} \label{indepcond2}
	FT(X(A=a,R=r)) \quad\stackrel{d}{=} \quad FT(X(A=a',R=r)) \quad \text{ for all } a,a',r.
\end{equation}
We construct a classifier (or a regressor) $f$ for $Y$ using the transformed data $FT(X)$ as the input in the training step. This immediately guarantees that, for the classifier $\widehat{Y} = f \circ FT$, the distribution of
\begin{equation}\label{eq:resolv} \widehat{Y}(A=a,R=r) \end{equation}
does not depend on the value of $a$ for any $r$.
Hence we achieve
population fairness no matter which classifier $f$ is used to obtain $\widehat{Y}$.
We will base the fair transformation $FT(x)$ on a presumed counterfactual model by defining, for $X(U=u)=x$, the transformation as
\begin{equation}\label{eq:fair-transform-resolvers} FT(X(U=u)) := X(A=0,U=u,R=R(U=u)).\end{equation}
That means we: (i) keep the latent quantile variables $U$ identical; (ii) set the protected attribute $A$ to its baseline value; (iii) keep the value of the resolvers equal to their values $R(U=u)$ under no intervention. The population level notion (\ref{RCFF1}) is guaranteed to be satisfied.
The adapted data $FT(X)$ can be interpreted as counterfactuals. If we believe the counterfactual model obtained when using Definition \ref{def:qpa}, we achieve the strong counterfactual notion of fairness, namely
\begin{equation}\label{eq:cf2} \widehat{Y}( FT(X(U=u))) = \widehat{Y}( FT(X(A=a,U=u))) \end{equation} for any value of $a$. That is the prediction $\widehat{Y}$ will be unchanged for the individual corresponding to $U=u$ under a do-intervention that is setting the protected attribute $A$ to its baseline value $A=0$ and keeping the resolving variables $R = r$ fixed. This follows directly from definition \eqref{eq:fair-transform-resolvers} of the data transformation, as $FT(X(U=u))$ is equal to $X(A=0,U=u, R = r)$ and also $FT(X(A=0,U=u, R = r))=X(A=0,U=u, R = r)$. But even if the counterfactual model is misspecified (and we cannot empirically verify this in principle), we still achieve the population fairness notion. We can hence summarize the goal as in the following table.
\medskip
\begin{center}
	\begin{tabular}{M{2.5cm}|M{3cm}|M{2.5cm} }
		 & population fairness & strong counterfactual fairness\\[1mm] 	\hline
		counterfactual model true &  \ding{52} & \ding{52}\\
		counterfactual model false & \ding{52} & \ding{56}  \\
	\end{tabular}
\end{center}
\medskip
To summarize: if the counterfactual model is right, we achieve
counterfactual fairness in the sense of Equation \eqref{RCFF3} with the data adaptation defined in Equation \eqref{eq:fair-transform-resolvers}. If the counterfactual model is wrong, we still achieve population fairness as in Equation \eqref{RCFF1}. An interesting special case is when there are no resolvers, $R = \emptyset$. In this case $FT(X) \ci A$ and \textit{demographic parity} $\widehat{Y} \ci A$ is guaranteed to hold, irrespective of the counterfactual model we are using.

When writing $X(U=u)=x$, note that $x$ and $u$ have a one-to-one correspondence for continuous random variables that permit a density. This means that $u$ is a deterministic function of the observed realization $x$. We get back to the need for randomization in the case of discrete random variables, as they do not have densities.

\subsection{Population level adaptation} \label{populationlevel}
The input of our procedure is the causal graph $\mathcal{G}$, a choice of resolving variables ${R}$ and data $(A_k, X_k, Y_k)_{k=1:n} = (A, X, Y)$. Even though we are describing the procedure on population level (meaning we are ignoring finite sample estimation errors) we still work with data samples to emphasize our counterfactual construction. We also assume that the densities $f_k(x_k \mid \pa(x_k))$ of the SEM in Equation \eqref{eq:scm} are known. The output of our procedure is the transformed data $FT(X)$ and $FT(Y)$. So far we have only considered adapting the covariates $X$. In the procedure we also adapt the response $Y$. This idea will be discussed shortly after the algorithmic procedure. In Section \ref{procedure} we explain how the procedure is carried out non-parametrically on sample level.
\begin{algorithm}
	\setstretch{1.1}
	\DontPrintSemicolon
	\KwIn{causal graph $\mathcal{G}$, density of the data generating mechanism $f(x_1,...,x_k) = \prod f(x_i \mid \pa(x_i))$, choice of resolving variables $R$, data $(A_k, X_k, Y_k)_{k=1:n} = (A, X, Y)$}
	\KwOut{adapted data $FT(X)$, $FT(Y)$}
	$FT(A_k) \gets 0$ for each $k$ \\
	$FT(R_k) \gets R_k$ for each $k$\\
	\For{$V \in \de(A) \setminus {R}$ in topological order}{
		using the density $f(v \mid \pa(v))$ obtain the inverse quantile function $g_V$ of $V$, such that $V = g_V(U^{(V)},\pa(V))$ \; \label{invquant}
		obtain the latent quantile $U^{(V)}_k$ of $V_k$ for each $k$\; \label{latentquant}
		obtain the transformed value $FT(V_k)$ using the transformed values of its parents (obtained in previous steps), by setting
		$FT(V_k) \gets g_V(U^{(V)}_k,\ FT(\pa(V_k)))$ \; \label{CFassign}
}
	\Return{$FT(X), FT(Y)$}
	\caption{{\sc Population Fairness Adaptation}}
	\label{algo:fairnesspopulation}
\end{algorithm}
We show that the procedure in Algorithm \ref{algo:fairnesspopulation} satisfies the desired fairness criteria in the following theorem:
\begin{theorem}[Population and strong resolved fairness]
\label{thm:populationlevel}
Let $FT(\cdot)$ be the transformation from Algorithm 1. Suppose $f$ is any classifier built based on the transformed data $FT(X,Y)$. Then we have that the classifier $\widehat{Y} = f \circ FT(\cdot)$ satisfies population resolved fairness, that is $$ \widehat{Y}(A = a, R = r) \quad \overset{d}{=} \quad \widehat{Y}(A = a', R =r) \quad \forall a,a',r.$$ Further, under the quantile preservation assumption (QPA) $\widehat{Y}$ satisfies strong resolved fairness,
$$ \widehat{Y}(A = a, R = r, U = u) \quad = \quad \widehat{Y}(A = a', R =r, U = u) \quad \forall a,a',r,u.$$
\end{theorem}
\noindent The main idea of the proof is to show that the $FT(\cdot)$ transformation is equivalent to the $do(A = 0, R = r)$ intervention and to use the counterfactual construction to show that the stronger fairness notion is satisfied. The full proof is given in Appendix \ref{appendix:theoremproof}.

\paragraph{Resolver-induced parity gap.} In the absence of resolving variable, the data adaptation of $X$ achieves the desired fairness criterion. Theorem \ref{thm:populationlevel} shows that $\widehat{Y} = f \circ FT$ satisfies the fairness notion \eqref{RCFF1} for any classifier $f$. However, in the presence of resolving variables, the setup is more subtle.
We look at a simple example of a  linear, additive regression model. The details of it are given in Table \ref{table:RIPGexample}.
\begin{table}
	\centering
	\begin{tabular}{>{\centering\arraybackslash} m{6cm} | c}
			Graphical model representation & Generating mechanism (SEM) \\[0.5ex]
			\hline
			\begin{tikzpicture}
		 [>=stealth, rv/.style={circle, draw, thick, minimum size=6mm}, rvc/.style={triangle, draw, thick, minimum size=10mm}, node distance=18mm]
		 \pgfsetarrows{latex-latex};
		 \begin{scope}
		 \node[rv] (1) at (-2,0) {$A$};
		 \node[rv] (2) at (0,1) {$X$};
		 \node[rv] (3) at (0,-1) {$R$};
		 \node[rv] (4) at (2,0) {$Y$};
		 \draw[->] (1) -- (2);
		 \draw[->] (1) -- (3);
		 \draw[->] (2) -- (4);
		 \node[] at (0,1.5) {};
		 \end{scope}
		 \end{tikzpicture} &
				$\begin{aligned}
				A &\gets \text{Bernoulli}(0.5) \\
				X &\gets \frac{1}{2}\mathbb{1}(A = 0) + \epsilon_X \\
				R &\gets \frac{3}{4}\mathbb{1}(A = 0) + \epsilon_R \\
				Y &\gets \frac{1}{2}X + \epsilon
				\end{aligned}$ \\
	\end{tabular}
	\caption{A full description of the example discussed in the text.}
	\label{table:RIPGexample}
\end{table}
Suppose we had data coming from this model. Assume that we want the variable $R$ to be resolving. Then our data adaptation would change the value of $X$ so that $$FT(X) = X - \frac{1}{2}\mathbb{1}(A = 0)$$
in order to remove the effect of $A$ from $X$. Suppose that, after this, we want to use the transformed values $FT(X)$ and $R$ to construct a predictor for $Y$. Since $Y$ can in fact be written as $$ Y = \frac{1}{2}X + \epsilon = \frac{1}{2}FT(X) + \frac{1}{4}\mathbb{1}(A = 0) + \epsilon$$ we can notice that $Y$ and $R$ are correlated. Furthermore, $Y \not\!\perp\!\!\!\perp R \mid FT(X)$. Therefore, if we linearly regress $Y$ onto $\lbrace FT(X),R\rbrace$ and obtain $\widehat{Y}$, then $R$ will have a non-zero coefficient. In fact, by using $R$, $\widehat{Y}$ will predict higher values for the $A = 0$ population. This is somewhat paradoxical, as $R$ has no causal effect on $Y$. The condition \eqref{RCFF1} will still hold, though. How do we reconcile this problem?

The example above illustrates that the condition \eqref{RCFF1} on its own is not sufficient to ensure the type of fairness we want. The difference between subpopulations should be at most the difference resulting from the causal effect of the resolvers. Therefore an additional strengthening of the criterion is necessary in the presence of resolving variables. We want to bound the maximum parity gap that can occur in the presence of the resolving variables.
\begin{definition}[Resolver-induced parity gap]	\label{def:RIPG}
	We say that a predictor $\widehat{Y}$ for $Y$ satisfies the resolver-induced parity gap with respect to a set of resolving variables $R$ if
	\begin{equation} \label{eq:RIPG}
	  \ex\big[\widehat{Y}(A = 0) - \widehat{Y}(A = 1)\big] \leq \ex\big[Y(A = 0) - Y(A = 0, R = R(1))\big].
	\end{equation}
\end{definition}
\noindent The LHS of the definition is the parity gap of our predictor $\widehat{Y}$. The RHS is the parity gap between the two groups explained only by the resolvers $R$. This is the maximum parity gap we wish to allow for $\widehat{Y}$. The main problem of the example above is, in some sense, that the labels $Y$ were not adapted jointly with $X$. Thus $Y$ still contained the effect of $A$. This effect was ultimately explained by $R$.

We argue that using the transformed labels $FT(Y)$ is one way to circumvent the problem. If we allow $\widehat{Y}$ to be a probability predictor (instead of a $\lbrace 0,1 \rbrace$ classifier), then $$f^\star(FT(X)) = \ex\big[FT(Y) \mid FT(X)\big]$$ satisfies
\begin{align*}
	\ex\big[f^\star(FT(X))\big] &= \ex\big[\ex[FT(Y) \mid FT(X)\big]\big] \\
	                    &= \ex\big[FT(Y)\big] \\
											&=\ex\big[Y(A = 0, R = R(a))\big],
\end{align*}
where the last equality comes from Theorem \ref{thm:populationlevel}. Therefore, $\widehat{Y} = f^\star \circ FT$ satisfies $$\ex\big[\widehat{Y}(A = a)\big] = \ex\big[Y(A = 0, R = R(a))\big]$$ from which it follows that this $\widehat{Y}$ satisfies the condition \eqref{eq:RIPG}. Finally, note that reasonable classifiers $f$ will converge to the population optimal prediction $f^\star$. For small sample sizes, the parity gap-condition is only fulfilled modulo sampling noise.   The reader might wonder why we work with probability predictions instead of $\lbrace 0, 1 \rbrace$ predictions. This is discussed in depth in Section \ref{methodformalisation}. A brief discussion related to the above argument is also given in Appendix \ref{appendix:whyprobability}.
\paragraph{Obtaining the (inverse) quantile function and latent quantiles.} In line \ref{invquant} of Algorithm \ref{algo:fairnesspopulation}, we obtain the inverse quantile function $g_V$ of $V$. More precisely, we first obtain the quantile function $Q_V(V ; \pa(V))$ using quantile regression forests \citep{qrf}, after which $g_V$ is obtained by inverting $Q_V(V ; \pa(V))$. The latent quantiles in line \ref{latentquant} are also obtained using quantile regression forests. Even though tree ensemble methods might perform worse in presence of heteroscedastic noise, we focus on this option because of its computational tractability. Alternative options for the quantile step are optimal transport methods, for instance \citep{qrot}.

\paragraph*{Discussion of the assumptions.}
The quantile preservation assumption (QPA) from Definition~\ref{def:qpa} is used for constructing a joint cross-world distribution. The assumption is equivalent to the equal noise assumption used in the NPSEM-IE framework of \cite{pearl2009}. This assumption has been much debated in the causal community in its different forms. See, for instance, \citep{dawid2000} or \citep{pearl2000}. The assumption is indeed not testable, not even in principle. We offer some thoughts on why QPA might be sensible in the fairness context:
\begin{enumerate}[(a)]
	\item If we consider two continuous distributions with cumulative functions $F_{X_0}$ and $F_{X_1}$ and $p \geq 1$, the Wasserstein distance $W_p(F_{X_0},F_{X_1})$ is minimised\footnote{We note that this minimisation is achieved uniquely in the case where $p > 1$, since the cost function is strictly convex.} by the optimal transport map $F_{X_1}^{-1} \circ F_{X_0} $, as shown by \cite{cuestaoptimal}. This mapping precisely represents quantile matching. We can see how QPA arises naturally as minimising the distance between counterfactual worlds.
	\item Consider two distributions which only differ by a shift in the mean, for example $N(\mu_1, \sigma^2)$ and $N(\mu_2,\sigma^2)$. In the null-case, when the means are the same, there is no difference between the subpopulations. When taking the limit to the null case, $\mu_2 \to \mu_1$, we recover the quantile matching property.
	\item The quantile preservation assumption ensures that we retain the original ordering of the values. Namely, if for a variable $V$ two individuals have equal values for all $\an(V)$, then QPA guarantees their counterfactual values $V(A=0)$ will retain the original ordering.
	\item In mathematical modelling, we often use noise to describe variations which are not explained by the data. This does not necessarily mean these are completely random, but could be a result of certain unobserved variables, which possibly cannot even be measured in practice. For instance, in the hypothetical intervention "What would have happened had this female been born a male?", there are a number of genetic, parenting and societal factors that would have remained the same.
	\item If we consider individuals with high quantiles $U$ (very successful individuals), it would be hard to argue how it could be fair that these quantiles do not stay the same in the counterfactual world.
	\end{enumerate}

\paragraph{Linear additive case.} The following theorem is intended to provide intuition about what the resolved fairness condition ensures in the simplest linear additive case.
\begin{theorem}[Strong resolved fairness for linear additive SEMs] \label{thm:linear}
Assume that we have an additive, linear structural equation model for variables $(A, X^{(1)}, ..., X^{(k)}, Y)$ and that $A \in \lbrace 0,1 \rbrace$ is a root node in the causal graph
\begin{align*}
	X^{(i)} \gets \sum_{V \in \pa_i} \beta^{(i)}_VV + \epsilon_i, \\
	Y \gets \sum_{V \in \pa({Y})} \beta_V^YV + \epsilon.
\end{align*}
The noise variables are assumed to be independent. Let $R$ be the set of the resolving variables. If $\widehat{Y} = \alpha_AA + \sum_{i = 1}^{k} \alpha_iX^{(i)}$ is a linear predictor for $Y$ then the strong resolved fairness condition \eqref{RCFF3} implies that
\begin{equation} \label{linearconstraint}
\sum_{j = 1}^{k} \alpha_j \times \ \left(\sum_{\substack{\text{paths } A \rightarrow X_j\\
                  \text{disjoint from } R}} \prod_{m \in \text{ path}} \beta_m \quad\right) - \alpha_A = 0.
\end{equation}
\end{theorem}
\begin{proof}
In the proof we suppress the notation $U = u$ to indicate that the quantiles are unchanged. Instead we say that condition \eqref{RCFF3} implies that all the noise variables $\epsilon_i$ remain unchanged under the $do(A = a, R = r)$ intervention. Hence we know that
\begin{align*}
	X^{(i)}(A = 1, R = r) - X^{(i)}(A = 0, R = r) = \\
	\sum_{V \in \pa_i} \beta_V^{(i)} (V(A = 1, R = r) - V(A = 0, R = r))
\end{align*}
By recursively expanding the sum on the RHS we obtain that
\begin{equation} \label{deltaX}
	X^{(i)}(A = 1, R = r) - X^{(i)}(A = 0, R = r) = \sum_{\substack{\text{paths } A \rightarrow X_j\\\text{disjoint from } R}} \prod_{k \in \text{ path}} \beta_k.
\end{equation}
Finally, condition \eqref{RCFF3} implies that
\begin{equation} \label{Yhat}
	\widehat{Y}(A = 1, R = r) - \widehat{Y}(A = 0, R = r) =  0.
\end{equation}
Expanding the expression \eqref{Yhat} using the identity \eqref{deltaX} gives precisely the constraint \eqref{linearconstraint}.
\end{proof}
\noindent Notice that resolved fairness is in the linear additive case equivalent to a single linear constraint on the coefficients of $\widehat{Y}$.

Lastly, we summarize the main advantages of our method:
\begin{enumerate}[(i)]
	\item it does not throw away information contained in $\de(A)$ which is potentially useful for prediction, as proposed in some previous works \citep{counterfactualfairness},
	\item it takes the causal perspective into account and offers interpretability of how and why fairness is achieved, ensuring that fairness criteria are not satisfied spuriously,
	\item it allows for a multitude of different fairness criteria, suitably adapted for different applications, reaching from demographic parity achieved when ${R} = \emptyset$, all the way to calibration which is often achieved when all variables are resolving, that is ${R} = {X}$.
\end{enumerate}

\section{Relation to existing work} \label{relation}
In this section we discuss the relation of fair adaptation to previous work on fairness.
\subsection{Observational notions of fairness}
For sake of brevity, we do not mention all the definitions of fairness proposed so far. We only review the most important observational notions. By \textit{observational notions} we refer to all notions that only focus on the observational distribution of the data, without taking the generating causal mechanism into account.
\begin{enumerate}[(i)]
	\item One of the first observational notions, called \textit{demographic parity}, goes all the way back to \citet{darlington1971}.
	\begin{definition}[Demographic parity] \label{dempar}
		A predictor $\widehat{Y}$ satisfies demographic parity if \begin{equation}\label{eq:dempar} \widehat{Y} \ci A.\end{equation}
	\end{definition}
	\noindent In the special context of binary predicted labels, $\widehat{Y}\in \{0,1\}$, demographic parity is equivalent to
	$
	\pr( \widehat{Y} = 1 \mid A = 0) = \pr(\widehat{Y} = 1 \mid A = 1)
	$. In words, this definition requires that our prediction is independent of the protected attribute. We now show that our population fairness criterion \eqref{RCFF1} is equivalent to demographic parity in the case when $A$ is a root node in the causal graph:
	\begin{proposition} \label{prop:demparity}
 Suppose that the protected attribute $A$ is a root node in the causal graph $\mathcal{G}$. If $\widehat{Y}$ is a binary predictor for the outcome $Y$, then we have that
	 \begin{equation}
		 \widehat{Y} \ci A \quad \iff \quad \widehat{Y}(A = a) \enskip\overset{d}{=} \enskip \widehat{Y} \enskip \forall a.
	 \end{equation}
 In words, if $A$ is a root node, then population fairness is equivalent to demographic parity.
 \end{proposition}
	\begin{proof}
 By applying the Action/Observation exchange rule (2nd rule of do-calculus), found in \citep{pearl2009} we have
	\begin{equation*} \label{rootnode}
		\widehat{Y}(A=a) \quad \overset{d}{=} \quad \widehat{Y} \mid A=a
	\end{equation*}
 Therefore, if $\widehat{Y}(A = a) \overset{d}{=} \widehat{Y} \enskip \forall a$, then for any $a,a'$,
 $$\widehat{Y} \mid A = a' \quad \overset{d}{=} \quad \widehat{Y}(A = a') \quad\overset{d}{=}\quad \widehat{Y}(A = a) \quad\overset{d}{=}\quad \widehat{Y} \mid A = a,$$
 implying demographic parity. The reverse implication works analogously.
 \end{proof}
	\item Another population definition of fairness is \textit{equality of odds}, first proposed by \citet{hardt2016}.
	\begin{definition}[Equality of odds] \label{EO}
		A predictor $\widehat{Y}$ satisfies equality of odds if
		\[ \widehat{Y} \ci A \mid Y.\]
	\end{definition} For  binary response $Y$ and prediction $\widehat{Y}$ (the original context in which it was proposed), equality of odds is equivalent to
	$
	\pr( \widehat{Y}=1 \mid Y = y,\  A = 0) = \pr(\widehat{Y}=1 \mid Y = y,\ A = 1)
	$ for $y \in \lbrace 0,1 \rbrace$. Only taking the equality above for $y = 1$ gives equality of opportunity. In words, this definition requires our prediction to be independent of the protected attribute, given the true outcome.
	\item The last observational notion we wish to mention is \textit{calibration}, recently discussed by \cite{chouldechova2017}. Here, it is assumed that $\widehat{Y}$ is an estimator of the true conditional  probability of $Y=1$. Calibration is defined as follows.
	\begin{definition}[Calibration] \label{suff}
		A prediction $\widehat{Y}$ satisfies calibration if \[ Y \ci A \mid \widehat{Y}.\]
	\end{definition}
For binary outcomes calibration is equivalent to
	$
	\pr( Y = 1 \mid \widehat{Y} = y,\  A = 0) = \pr( Y = 1 \mid \widehat{Y} = y,\ A = 1)
$ for $y \in [0,1]$.
	Calibration states that, given our prediction, the protected attribute should not provide us with additional information about the true outcome.
\end{enumerate}

\subsection{Adaptation and observational criteria} \label{relationobservational}
We discuss the relation between the observational criteria and our adaptation method. Consider the examples given in Table \ref{table:relationtable}, in which we consider a classifier $\widehat{Y}$ to be a function of the adapted data $FT({X})$. For understanding the examples, it suffices to think of a non-resolving variable as adapted to contain no effect of $A$. In the table we discuss the possibility of $Y$ being a resolving variable, which might seem confusing. By $Y$ being resolving we simply mean that the true outcome is considered fair as it is. In the given examples we consider $X$ to be a single feature, although the conclusions remain the same for multiple features.
\begin{table}\centering
	\caption{Examples that describe the intrinsic relation of observational criteria to counterfactual fairness.}\medskip\small
	\begin{tabular}{|M{2cm}|M{5cm}|M{5cm}|}
		\hline
		Example & Causal graph & Observational criterion achieved \\[1mm]\hline
		(a) & \begin{tikzpicture}
		[>=stealth, rv/.style={circle, draw, thick, minimum size=6mm}, rvc/.style={triangle, draw, thick, minimum size=7mm}, node distance=18mm]
		\pgfsetarrows{latex-latex};
		\begin{scope}
		\node[rv] (1) at (-2,-0.5) {$A$};
		\node[rv] (2) at (0,-0.5) {$X$};
		\node[rv] (3) at (2,-0.5) {$Y$};
		\draw[->] (1) -- (2);
		\draw[->] (2) -- (3);
		\end{scope}
		\node[] at (0,-1.25) {$X$ not resolving};
		\node[] at (0,0) {};
		\end{tikzpicture} & \Large $\widehat{Y} \ci A$\\\hline
		(b) & \begin{tikzpicture}
		[>=stealth, rv/.style={circle, draw, thick, minimum size=6mm}, rvc/.style={triangle, draw, thick, minimum size=7mm}, node distance=18mm]
		\pgfsetarrows{latex-latex};
		\begin{scope}
		\node[rv] (1) at (-2,-0.5) {$A$};
		\node[rv] (2) at (0,-0.5) {$Y$};
		\node[rv] (3) at (2,-0.5) {$X$};
		\draw[->] (1) -- (2);
		\draw[->] (2) -- (3);
		\end{scope}
		\node[] at (0,-1.25) {$Y$ not resolving};
		\node[] at (0,0) {};
		\end{tikzpicture} & \Large $\widehat{Y} \ci A$\\\hline
		(c) & \begin{tikzpicture}
		[>=stealth, rv/.style={circle, draw, thick, minimum size=6mm}, rvc/.style={triangle, draw, thick, minimum size=7mm}, node distance=18mm]
		\pgfsetarrows{latex-latex};
		\begin{scope}
		\node[rv] (1) at (-2,-0.5) {$A$};
		\node[rv] (2) at (0,-0.5) {$Y$};
		\node[rv] (3) at (2,-0.5) {$X$};
		\draw[->] (1) -- (2);
		\draw[->] (2) -- (3);
		\end{scope}
		\node[] at (0,-1.25) {$Y$ resolving};
		\node[] at (0,0) {};
		\end{tikzpicture} & \Large $\widehat{Y} \ci A \mid Y$\\\hline
		(d) & \begin{tikzpicture}
		[>=stealth, rv/.style={circle, draw, thick, minimum size=6mm}, rvc/.style={triangle, draw, thick, minimum size=7mm}, node distance=18mm]
		\pgfsetarrows{latex-latex};
		\begin{scope}
		\node[rv] (1) at (-2,-0.5) {$A$};
		\node[rv] (2) at (0,-0.5) {$X$};
		\node[rv] (3) at (2,-0.5) {$Y$};
		\draw[->] (1) -- (2);
		\draw[->] (2) -- (3);
		\end{scope}
		\node[] at (0,-1.25) {$X$ resolving};
		\node[] at (0,0) {};
		\end{tikzpicture} & \Large $Y \ci A \mid \widehat{S}$\\\hline
	\end{tabular}
	\label{table:relationtable}
\end{table}
We first provide a formal statement about Table \ref{table:relationtable}.
\begin{theorem}[Fairness criteria] \label{toytheorem}
Assume that for examples (a)-(c) from Table \ref{table:relationtable} we are building a classifier $\widehat{Y}$ based on appropriately adapted data $FT(X,Y)$ which satisfies the condition (\ref{RCFF1}) for the choice of resolvers $R$ given in the table. In example (d) we are building a predictor of the positive probability $S(x) = \pr(Y = 1 \mid X = x)$. For the given examples we have the following:
\begin{enumerate}
	\item [(a)] for $\widehat{Y}$ built based on $FT(X, Y)$ we have $\widehat{Y} \ci A$
	\item [(b)] for $\widehat{Y}$ built based on $FT(X, Y)$ we have $\widehat{Y} \ci A$
	\item [(c)] for $\widehat{Y}$ built based on $FT(X, Y) = (X,Y)$ we have $\widehat{Y} \ci A \mid Y$
	\item [(d)] under the additional assumption that our predictor $\widehat{S}$ built based on $FT(X, Y) = (X,Y)$ equals the true positive probability $S(x) = \pr(Y = 1 \mid X = x)$ we have that $Y \ci A \mid \widehat{S}$
\end{enumerate}
\end{theorem}
\begin{proof}
Consider the following:
\begin{enumerate}
	\item [(a,b)]  In both examples the values of $X$ and $Y$ are transformed to $FT(X, Y)$ which are independent of $A$, by condition (\ref{RCFF1}). Any predictor $\widehat{Y}$ which is a function of $FT(X)$ must also be independent of $A$.
	\item [(c)] Consider the following graph where the predictor $\widehat{Y}$ is included in the causal graph
	 \begin{figure}[H] \centering
		\begin{tikzpicture}
		[>=stealth, rv/.style={circle, draw, thick, minimum size=5mm}, rvc/.style={triangle, draw, thick, minimum size=5mm}, node distance=15mm]
		\pgfsetarrows{latex-latex};
		\begin{scope}
		\node[rv] (1) at (-3,0) {$A$};
		\node[rv] (2) at (-1,0) {$Y$};
		\node[rv] (3) at (1,0) {$X$};
		\node[rv] (4) at (3,0) {$\widehat{Y}$};
		\draw[->] (1) -- (2);
		\draw[->] (2) -- (3);
		\draw[->] (3) -- (4);
		\end{scope}
		\end{tikzpicture}
	\end{figure}
	Clearly we have that $Y$ d-separates $A$ and $\widehat{Y}$ and the conclusion follows.
	\item [(d)] In this example we can view the causal representation to be expanded as follows:
\begin{figure}[H] \centering
	\begin{tikzpicture}
	[>=stealth, rv/.style={circle, draw, thick, minimum size=5mm}, rvc/.style={triangle, draw, thick, minimum size=5mm}, node distance=15mm]
	\pgfsetarrows{latex-latex};
	\begin{scope}
	\node[rv] (1) at (-3,0) {$A$};
	\node[rv] (2) at (-1,0) {$X$};
	\node[rv] (3) at (1,0) {$S$};
	\node[rv] (4) at (3,0) {$Y$};
	\draw[->] (1) -- (2);
	\draw[->] (2) -- (3);
	\draw[->] (3) -- (4);
	\end{scope}
	\end{tikzpicture}
\end{figure}
where $S(x) = \pr(Y = 1 \mid X = x)$ is the true positive probability. Under the additional assumption $\widehat{S} = S$, we have that $\widehat{S}$ d-separates $A$ and $Y$.
\end{enumerate}
\end{proof}
We can now clarify the core ideas of observational notions discussed in this section. Note that in the toy examples from Table \ref{table:relationtable} we have that:
\begin{enumerate}
	\item [(a,b)] Demographic parity is achieved when $X$ or $Y$ are considered to be non-resolving. We can see that in some sense demographic parity is a criterion that requires us to treat all subpopulations as exactly the same, regardless of what is observed in the data.
	\item [(c)] Equality of odds is achieved when $Y$ is considered to be resolving. In this case the adaptation procedure does not change the values of $X,Y$. The idea that the true outcomes $Y$ are fair is in the heart of this notion.
	\item [(d)] Calibration can be\footnote{It might be valuable to note that calibration does not necessarily have to arise in this case. The criterion is still dependent on how we build our classifier.} achieved when $X$ is considered to be resolving. In this case our adaptation procedure does not change the values of $X,Y$. Calibration is a criterion that ensures we do not discriminate any subpopulation beyond the differences observed in the data. Calibration should often come as a result of a good maximum utility predictor.
\end{enumerate}

\subsection{Mediation and path-specific effects} \label{pathspecmethods}
It is important to mention the connection of our work with some previous works \citep{shpitser2018, pathspecific}. In particular, \cite{shpitser2018} start with the joint distribution $p(X,Y)$ and define discrimination as $\phi(p(X,Y))$, where $\phi$ is some functional of the distribution. After that, their goal is to find another distribution $p^\star(X,Y)$ which is close to the original $p(X,Y)$ and satisfies $|\phi(p^\star(X,Y))| \leq \epsilon$. One possible choice of $\phi$ they work with is the  natural direct effect (NDE) which is defined as \[ NDE = E\big[  Y(A = a, {R} = R(a')) - Y(A=a')    \big]  .\]
The NDE can be interpreted as the total causal effect of the protected attribute on the outcome that does not go through resolving variables. If we use the transformed distribution $FT(X,Y)$ as the $p^\star(X,Y)$, then we can relate their approach to our method via the following proposition
\begin{proposition} \label{prop:NDE}
	Suppose that condition~\eqref{RCFF1} holds, that is $$\widehat{Y}(A = a, {R} = {r}) \quad \stackrel{d}{=}\quad  \widehat{Y}(A = a', {R} = {r})\;\; \forall r.$$
	Then it follows that \[E\big[  \widehat{Y}(A = a, {R} = R(a')) - \widehat{Y}(A=a')    \big] = 0  .\]
\end{proposition}
\noindent The short proof is given in Appendix \ref{appendix:NDE}. The proposition shows that condition \eqref{RCFF1}, which our method achieves, is sufficient (but not necessary) for the NDE to vanish for the transformed distribution.

\section{Practical aspects and extensions} \label{methodformalisation}
After explaining the main ideas and fairness criteria our method achieves, we turn to discussing the practical aspects and extensions of our method.

\subsection{Categorical (and discrete) variables} \label{categorical}
An important practical aspect of our method is dealing with variables that take values on a discrete domain. There is an immediate problem we encounter in this case. If we think about the mapping $u \rightarrow V(U=u)$, we can see that different values of $u$ can correspond to the same value of $V(U = u) = v$, that is the mapping is no longer injective (as opposed to the continuous case). To summarize, the conditional distribution $U \mid V = v$ is deterministic in the continuous case, and non-deterministic in the discrete case.

\paragraph{Ordered categorical and discrete variables.} As a starting point, we describe our method for a binary variable $V \in \lbrace 0,1 \rbrace$. Consider the probabilities $p_0 := \pr(V = 0 \mid \pa(V),\ A = 0)$ and $p_0' := \pr(V = 0 \mid \pa(V),\ A = 1)$. Assume without losing generality that $V=0$. Then we compute the transformed value $FT(V)$ as:
\begin{itemize}
	\item if $p_0' \leq p_0$ then $FT(V)= 0$
	\item if $p_0' > p_0$ then
	\begin{equation*}
	\setstretch{1.5}
	FT(V) = \begin{cases}
	0 \quad \text{with probability } \frac{p_0}{p_0'}\\
	1 \quad \text{with probability } \frac{p_0'-p_0}{p_0'}\\
	\end{cases}
	\end{equation*}
\end{itemize}
We need to generalise this approach to non-binary, discrete variables $V$. Suppose now that $V$ takes values in $\lbrace 1,...,m \rbrace$. Similarly as above, define ${p} = (p_1,...,p_m)$ and $ {p}' = (p'_1,...,p'_m)$ where:
\begin{align}
		p_i &:= \pr(V = i \mid \pa(V),\ A = 0) \label{eq:ctfdist1} \\
	p'_i &:= \pr(V = i \mid \pa(V),\ A = 1 ) \label{eq:ctfdist2}
\end{align}
These probabilities can, for example, be estimated using probability random forests \citep{malley2012probability}. Motivated by the quantile matching assumption, which arises as a solution that induces minimal change in the counterfactual world, we want to find a joint density for ${p},\ {p}'$ that minimises some transport cost. This can be done by solving the following optimisation problem:
\begin{equation}
	\begin{gathered}
	\min_{\Pi \in \mathbbm{R}^{m \times m}}\quad \text{Tr}(\Pi C) \label{eq:opt}\\
	\begin{aligned}
	\textup{s.t.}\quad \sum_{j=1}^{m} \Pi_{ij}  &=  p_i & \forall i \in \lbrace 1,...,m \rbrace\\
	\sum_{i =1}^{m} \Pi_{ij} & =  p'_j & \forall j \in \lbrace 1,...,m \rbrace\\
	\end{aligned}
	\end{gathered}
\end{equation}
where the cost matrix $C$ has entries $C_{ij} = |i-j|^p$. The exact value of $p$ does not really matter, since any $p > 1$ will give the same (unique) solution. When $V = i$, we sample $FT(V)$ from the distribution given by $\widehat{\Pi}_i$, the $i^\text{th}$ row of the optimal transport matrix. In particular $\widehat{\Pi}_i$ needs to be normalised, and we let $F_{\widehat{\Pi}_i}$ be the corresponding cumulative distribution function. We then have
\begin{equation} \label{eq:discretesampling}
	FT(V) = F^{-1}_{\widehat{\Pi}_i}(U), \text{ where } U \sim U[0,1]
\end{equation}
Notice that $FT(V)$ is not necessarily deterministic. It can happen that $\widehat{\Pi}_i$ has multiple non-zero entries, meaning that the value $V=i$ is coupled with multiple counterfactual outcomes. The reason for this was already mentioned, namely the fact that the conditional distribution $U \mid V = v$ is non-deterministic in the discrete case.

\paragraph{Unordered categorical variables.}  Let $V$ be categorical and unordered. We first obtain an ordering for it. Suppose $V$ takes values $C_1,...,C_l$. We then find a bijection $\sigma: \lbrace C_1,...,C_l \rbrace \rightarrow \lbrace 1,...,l \rbrace$ such that
\begin{equation} \label{marginallyincreasing}
	\sigma(C_i) \leq \sigma(C_j) \implies \pr(Y = 1 \mid V = C_i,\ A = 0) \leq \pr(Y=1 \mid V = C_j,\ A = 0)
\end{equation}
Then simply define $V' = \sigma(V)$ and use it as a replacement for $V$. Note that the condition (\ref{marginallyincreasing}) implies that the marginal probability $\pr(Y = 1 \mid V' = v,\ A = 0)$ is increasing in $v$. Implicitly, we assume that the same holds for $A = 1$. That is, we assume that $\pr(Y = 1 \mid V' = v,\ A = 1)$ is also increasing in $v$. Then we can again apply the approach used for discrete variables.

If there is no meaningful ordering, or we have reason to believe that imposing an ordering does not make sense, a slightly different approach is needed. We define ${p},\ {p}'$ the same way as above, with $p_i = \pr(V = C_i \mid \pa(V),\ A = 0)$ and $p'_i := \pr(V = C_i \mid \pa(V), A = 1)$. We again solve the optimisation problem (\ref{eq:opt}), but with a different cost matrix $C$, namely $C_{ij} = \mathbb{1}(i \neq j)$. When $V = C_i$, the distribution of $FT(V)$ is given by the (appropriately normalized) $i^\text{th}$ column of $\widehat{\Pi}$.

\subsection{Inherent limitation of the discrete case} \label{ctot}
In Section \ref{populationlevel} we gave an optimal transport interpretation of the quantile preservation assumption (QPA). In particular, we state that for two random variables $X,Y$ with distribution functions $F_X,F_Y$ the Wasserstein distance $W_p(X,Y)$ is minimised by matching the quantiles, that is using the optimal transport map given by $F^{-1}_Y \circ F_X$. This map is the optimal transport map for every $p \geq 1$ and also a unique optimal transport map for $p > 1$, since the cost function then becomes strictly convex.

The quantile matching is the \textit{greedy solution}. Note that this approach also extends to the discrete case - a greedy solution\footnote{A reader familiar with optimal transport will recognize that here we are talking about solutions satisfying the $c$-monotonicity property.} is optimal whenever the cost function is strictly convex \citep[chap.~2]{santambrogio2015}. However, there is a major difference between the continuous and the discrete case.

In the continuous case, using the quantile preservation assumption, we are able to compute the counterfactual values exactly. Richness of the ambient space allows for the optimal transport map to be deterministic, whereas in the discrete setting this is never the case. The solution of the problem (\ref{eq:opt}) gives us a non-deterministic distribution over the counterfactual outcomes. The reason for this is that it is impossible to distinguish individuals which have the same value of a variable $V$ - in some sense, the information coming from the quantiles is compressed.

A possible solution which first comes to mind is to perhaps take the expectation over this randomness. But even if we consider the simplest example, we run into a problem. Consider using a single binary predictor $X \in \lbrace 0,1\rbrace$ distributed as $\pr(X = 1 \mid A = 0) = 0.5$ and $\pr(X = 1 \mid A = 1) = 0.4$. Suppose that the outcome $Y$ simply equals $X$. After solving the optimal transport problem, all individuals with $A = 1, X = 0$ would have the counterfactual distribution
\begin{equation*}
	\pr(X(A=0) = 1 \mid A = 1, X = 0) = 1 - \pr(X(A=0) = 0 \mid A = 1, X = 0) = \frac{1}{6}
\end{equation*}
and all other individuals would retain the values they have. But when taking the expectation over this randomness, we have that
\begin{equation*}
	\ex[ X(A=0) \mid A= 1, X = 0] = \frac{1}{6}
\end{equation*}
meaning we get no additional information to distinguish between individuals with $A = 1, X = 0$. To treat everyone equally, we would have to either assign everyone $X = 1$ or $X = 0$, neither of which options is desirable. Therefore, we use \textit{randomisation}, which in this case chooses a "lucky" 1/6 of the individuals with $A=1, X =0$ and sets their counterfactual values $X(A=0)$ to $1$. For some regression applications integrating outcomes over different counterfactual worlds is meaningful. In that case, taking expectation over the assignment randomness might be sensible. For classification, where labels are either $0$ or $1$, this might fail, as shown above.

Further, consider two variables $X_1 \sim N(0,1)$ and $X_2 = \mathbb{1}(X_1 > 0)$. If we only have the variable $X_2$ available, then it is impossible to distinguish between individuals that have $X_2 = 1$. However, if we use $X_1$ instead, then no two individuals will be the same - we will always be able to distinguish them. When going from $X_1$ to $X_2$, we see that \textit{quantiles are compressed} and they can only be determined up to an interval. This causes the counterfactual value to be non-deterministic.

Finally, we clarify the difference in approach for different types of variables taking values on discrete domains. There are two cases we consider:

\begin{enumerate}[(a)]
	\item Discrete and ordered categorical variables:
		\subitem We solve the optimisation problem (\ref{eq:opt}) with the cost matrix $C_{ij} = |i-j|^p$ corresponding to $\ell_p$-loss. The cost matrix reflects the fact that, due to an inherent ordering of the values, we wish to penalise larger changes more. For any $p>1$ the greedy solution is the unique optimal solution.
	\item Unordered categorical variables:
		\subitem We also solve the optimisation problem (\ref{eq:opt}), but with the cost matrix $C_{ij} = \mathbb{1}(i \neq j)$. This cost matrix corresponds to $\ell_0$-loss. Since in this case we do not have an inherent ordering structure of the values, we penalise all changes equally. The optimal solution in this case is not unique.
\end{enumerate}

\subsection{Sample-level adaptation} \label{procedure}
Let $\mathcal{G}$ be the causal graph and let $R$ be a choice of resolving variables. Further, let $f(V \mid \pa(V))$ be the density corresponding to variable $V$. Let $g(\pa(V),\ U^{(V)})$ represent the inverse quantile function of the distribution of $V$, that is $V = g(\pa(V),\ U^{(V)})$. Sample level fair adaptation is given in Algorithm \ref{algo:fairnessadaptation}.
\begin{algorithm}
	\setstretch{1.05}
	\DontPrintSemicolon 
	\KwIn{Data $(A_k,{X}_k,Y_k)_{k=1:n}$, causal graph $\mathcal{G}$, choice of resolving variables $R$.}
	\KwOut{Adapted data $FT(A_k,\ X_k,\ Y_k)_{k=1:n}$ }
	\For{$V \in \de(A) \setminus {R}$ \text{ in topological order}}{
	\uIf{$V$ continuous}{
		estimate the quantiles $(\widehat{U}^{(V)}_k)_{k=1:n}$ of $V$ in the distribution $f(V \mid \pa(V))$ using quantile regression on the data $(V_k,\ \pa(V_k))_{k=1:n}$\; \label{algo:compquant}
		using $(V_k,\ \pa(V_k),\widehat{U}^{(V)}_k)_{k=1:n}$ obtain an estimator $\widehat{g}(\pa(V),\ U^{(V)})$ of $g(\pa(V),\ U^{(V)})$
	}
	\Else{
		\Case{\textnormal{\textbf{1.}} $V$ discrete and $V\neq Y$}{
			estimate the probability distributions $\widehat{{p}}(\pa(V_k))_{k=1:n}$ as in Equations \eqref{eq:ctfdist1}-\eqref{eq:ctfdist2}\;
			obtain the transformed probability distributions $\widehat{{p}}(FT(\pa(V_k)))_{k=1:n}$ \;
			$\forall k$ solve the optimal transport problem (\ref{eq:opt}) between $\widehat{{p}}(\pa(V_k))$ and $\widehat{{p}}(FT(\pa(V_k)))$ with $\ell_p$-loss to get $(\widehat{\Pi}^k)_{k=1:n}$
		}
		\Case{\textnormal{\textbf{2.}} $V = Y$}{
			perfrom \textbf{case 1.} restricted to the training set
		}
	}
		\For{all $k$ with $A_k = 1$ and $V_k$ known}{
			\uIf{$V$ continuous}{
					$FT(V_k) \gets \widehat{g}(FT(\pa(V_k)),\ \widehat{U}^{(V)}_k)$
									}
				\Else{
					$FT(V_k) \gets$ sample from the distribution $\widehat{\Pi}^k_{V_k}$ as in Equation (\ref{eq:discretesampling})
				}
		}
  }
	\Return{$FT(A_k,\ X_k,\ Y_k)_{k=1:n}$}
	\caption{{\sc Fairness Adaptation}}
	\label{algo:fairnessadaptation}
\end{algorithm}
Notice that our procedure treats the response $Y$ separately. The only reason for this is that $Y$ is unavailable on the test set.
The quantile regression step can be done either using random forests \citep{qrf} or by using an optimal transport approach \citep{qrot}.

\subsection{About the training step}
To construct a useful predictor $\widehat{Y}$ we must make use of the labels $Y$. The adapted labels $FT(Y)$ are available as an output of our fair adaptation procedure. The \textit{resolver-induced parity gap} was introduced in Definition \ref{def:RIPG} after convincing ourselves that condition \eqref{RCFF1} is not sufficient on its own (in the presence of resolving variables).

It is important to mention that, in the case of resolving variables, data from different counterfactual worlds should not be used. For instance, one should not use the original labels $Y$ with the transformed covariates $FT(X)$. If doing so, the parity gap condition \eqref{eq:RIPG} might not be satisfied. We refer the reader back to example given in Section \ref{populationlevel} and Definition \ref{def:RIPG}. It is also not advisable to leave out variables in the training procedure. We discuss two options for the training step. These are:
\begin{enumerate}[(A)]
	\setstretch{1}
	\item train the classifier on the original data $(A_k,\ {X}_k,\ Y_k)_{k=1:n}^\text{train}$ \label{methodone}
	\item train with the adapted data and the adapted labels $FT(A_k,\ X_k,\ Y_k)_{k=1:n}^\text{train}$\label{methodtwo}
\end{enumerate}
For both methods above, the adapted test data $FT(A_k,\ X_k,\ Y_k)_{k=1:n}^\text{test}$ should be used to produce the predictions for the test set. We have shown in Section \ref{populationlevel} that method \eqref{methodtwo} is sensible for satisfying the condition \eqref{eq:RIPG}. Similar reasoning can be used for method \eqref{methodone}, that is if $f^\star$ is such that $$f^\star(X) = \ex \big[Y \mid X \big]$$ then for $\widehat{Y} = f^\star \circ FT$ we have that
\begin{align*}
	\ex \big[\widehat{Y}(A = 0, R = R(a)) \big] &= \ex \big[f^\star(X(A = 0, R = R(a)))\big]\\
																		&= \ex\big[\ex\big[Y(A = 0, R = R(a)) \mid X(A = 0, R = R(a))\big]\big]\\
																		&= \ex\big[Y(A = 0, R = R(a))\big]
\end{align*}
from which it follows that this $\widehat{Y}$ also satisfies the condition \eqref{eq:RIPG}.
The better of the two options can be chosen via cross-validation as the one with the best fairness-accuracy trade-off. For experimental results in Section \ref{experimental} we by default use training method \eqref{methodtwo}.

\subsection{Method extensions}
There are two methodological extensions of our approach that we briefly discuss, leaving out some of the detail.
\paragraph*{Is there really a baseline?} So far we have considered the subpopulation $A=0$ to be the baseline. This choice is somewhat arbitrary. We briefly comment on the implications of choosing a baseline.

Firstly, the choice of the baseline can influence the optimism of our predictor. Imagine that we are trying to predict recidivism on parole, with race being the protected attribute. If we adapt the data using the white subpopulation as baseline, then our predictor will be much more optimistic, meaning that it will predict fewer recidivism outcomes than it would have in the case of choosing the black subpopulation as the baseline.

Secondly, if we have discrete variables, then our procedure will include some randomisation. There is no randomisation for the baseline population, but there is for the rest. If the baseline is the advantaged group, then randomisation can serve as \textit{positive discrimination} and might be seen as acceptable. However, we might want to consider an approach in which both subpopulations are randomised equally. We briefly discuss how we might split the burden of randomisation between the subpopulations.

\paragraph{A non-baseline approach.} We previously discussed adapting the data to the $A = 0$ baseline using Algorithm \ref{algo:fairnessadaptation}, which gives us the pre-processed version of the data, which we here label $\widetilde{{X}}^{A = 0}$. Of course, the same procedure can be applied to obtain the version corresponding to the $A =1$ baseline, which we label $\widetilde{{X}}^{A = 1}$. Then we can use the following approach:
\begin{enumerate}
	\item Obtain $(\widetilde{{X}}^{A = 0}, \widetilde{Y}^{A=0})$ and $(\widetilde{{X}}^{A = 1}, \widetilde{Y}^{A=1})$ using Algorithm \ref{algo:fairnessadaptation}.
	\item Concatenate the two versions to obtain
	\begin{equation*}
			{X}^\star = (\widetilde{{X}}^{A = 0},\widetilde{{X}}^{A = 1}).
		\end{equation*}
	\item Build predictors $\widehat{\pi}^{A=0}({x}^\star),\widehat{\pi}^{A=1}({x}^\star)$ that estimate the probabilities $\pr(\widetilde{Y}^{A=0} = 1 \mid {X}^\star = {x}^\star)$, $\pr(\widetilde{Y}^{A=1} = 1 \mid {X}^\star = {x}^\star)$ respectively.
	\item For any test observation with ${X}^{\star}_{\text{test}} = {x}^{\star}_{\text{test}}$ return the predicted probability of
	$$ \widehat{\pi}({x}^{\star}_{\text{test}}) = \frac{\widehat{\pi}^{A=0}({x}^{\star}_{\text{test}})+\widehat{\pi}^{A=1}({x}^{\star}_{\text{test}})}{2}.$$
\end{enumerate}
We offer an interpretation of the approach above. First we combine the information from the two worlds in which $A = 0$ and $A = 1$. We then use the joint information to predict probabilities of positive outcomes in both of these worlds. In the final step, we combine the probabilities from the two worlds by simply taking the mean probability. In this way, we obtain probability estimates for positive outcomes, which can then be used to construct a classifier by thresholding.

\paragraph{Edge specific extension.}
We quickly mention another possible extension of our method. So far we discussed resolving variables, on which the effect of $A$ is deemed fair. Sometimes deciding if a variable is resolving might not be straightforward. For example, we might see the causal effect of several causal parents as being fair, but of several as being unfair. In some sense, our method so far was focused on the nodes in the causal graph $\mathcal{G}$. An extension of this, which focuses more on specific edges in $\mathcal{G}$ is possible. A short discussion of this idea, motivated by an example, is given in Appendix \ref{appendix:edgeextension}.

\section{Experimental results} \label{experimental}
An implementation of our method, which uses tree ensembles for the quantile learning step, is available as the \texttt{fairadapt} package on CRAN. Our experimental results consist of two parts. In the first part, we look at synthetic examples which demonstrate how our methodology offers flexibility compared to possibly prohibitively strong demographic parity. In the second part, we look at the method performance on two real world datasets, comparing it to several different baseline methods.

\subsection{Measures of fairness and performance}
Before displaying the experimental results, we discuss all the measures that are used for assessment of our classifiers. For measuring performance, we report on accuracy in the simple classification tasks. It is, however, sometimes desirable to work with probability predictions, in which case we report the area under the receiver operator characteristic (AUC).
\paragraph{Fairness measures.}
There are two fairness measures we use. To assess demographic parity, we use the \textit{parity gap}, defined as $\pr(\widehat{Y} = 1 \mid A = 0) - \pr(\widehat{Y} = 1 \mid A = 1)$. When dealing with probability predictions, we simply report the parity gap at the 0.5 threshold.

In order to assess calibration, we need a measure for it. We introduce the \textit{$k$-level calibration score}. Suppose we have the predicted positive probabilities $\pr(\widehat{Y}=1 \mid {X} = {x})$ and the true labels $Y$.
We start by splitting the individuals with $A = 0$ into $k$ groups, based on the predicted probability of $\pr(\widehat{Y} = 1 \mid {X} = {x})$. In particular, if $\pr(\widehat{Y}=1 \mid {X} = {x}) \in [\frac{i}{k}, \frac{i+1}{k})$, then the individual is assigned to group $G_i$. In each group we compute the mean of the true outcomes $Y$ for that group, $\ex[Y \mid G_i]$, which is simply the proportion of positive outcomes in the group $G_i$. Assume that the vector $\pmb{c}^{A=0}$ contains these proportions for each group. We compute $\pmb{c}^{A=1}$ for the $A = 1$ population in the same way. Then the $k$-level calibration score is defined as $$\frac{1}{k}||\pmb{c}^{A=0}-\pmb{c}^{A=1}||_1 .$$ Note that for a well-calibrated score, this measure should be small. If calibration is satisfied, as $k \to \infty$, the $k$-level calibration score tends to 0.

\subsection{From parity to calibration}
Earlier in the text we claimed that our method can offer fairness criteria which are between demographic parity and calibration. Namely, Theorem \ref{toytheorem} shows that in a simple case demographic parity is achieved when none of the variables are resolving, and that calibration can be achieved if all of the variables are resolving. Depending on the choice of the resolving variables, we can interpolate between these two notions of fairness. Roughly speaking, the larger the resolving set is, the larger the effect of $A$ is in the data. In that case, the predictor $\widehat{Y}$ is closer to the unconstrained $\widehat{Y}^{\text{max-util}}$ predictor, meaning that we are closer to satisfying calibration. The smaller the resolving set is, the smaller the effect of $A$ is, meaning that we are closer to demographic parity.
\begin{table}
	\centering
	\begin{tabular}{c| c}
			Synthetic A & Synthetic B \\[0.5ex]
			\hline
			$\begin{aligned}
				A &\gets \text{Bernoulli}(0.5) \\
				X_i &\gets -\frac{A}{4} + \frac{1}{8} + \epsilon_i \quad \text{for } i \in \lbrace 1,...,5 \rbrace \\
				Y &\gets \text{Bernoulli}(\text{expit}(\sum_{i=1}^5 X_i))
				\end{aligned}$ &
				$\begin{aligned}
				A &\gets \text{Bernoulli}(0.5) \\
				X_i &\gets -\frac{A}{4} + \frac{1}{8} + \epsilon_i \quad \text{for } i \in \lbrace 1,2 \rbrace \\
				X_3 &\gets \frac{1}{4}X_2 + \epsilon_3 \\
				Y &\gets \text{Bernoulli}(\text{expit}(\sum_{i=1}^3 X_i))
				\end{aligned}$ \\
	\end{tabular}
	\caption{Structural equation models for the two synthetic examples A and B. All the noise variables $\epsilon_i$ are independent and expit$(x) = \frac{e^x}{1+e^x}$.}
	\label{table:SEMtable}
\end{table}

We demonstrate this by looking at two synthetic examples, with their structural equation models given in Table \ref{table:SEMtable}.
All the noise terms $\epsilon_i$ are independent $N(0,1)$ variables and $\text{expit}(x) = \frac{e^x}{1+e^x}$. In words, $Y$ follows a logistic regression model based on the $X_i$'s. The causal graphs of the two synthetic examples are given in Figure \ref{fig:gapcalib}.
\begin{figure} \centering
	\begin{tikzpicture}
	[>=stealth, rv/.style={circle, draw, thick, minimum size= 5mm}, rvc/.style={triangle, draw, thick, minimum size=5mm}, node distance=20mm]
	\pgfsetarrows{latex-latex};
	\begin{scope}
	\node[rv] (0) at (-2.5,0) {$A$};
	\node[rv] (1) at (-1.5,1) {$X_1$};
	\node[rv] (2) at (0,1) {$X_2$};
	\node[rv] (3) at (1.5,1) {$X_3$};
	\node[rv] (4) at (-1,-1) {$X_4$};
	\node[rv] (5) at (1,-1) {$X_5$};
	\node[rv] (6) at (2.5,0) {$Y$};
	\draw[->] (0) -- (1);
	\draw[->] (0) edge[bend right = 10] (2);
	\draw[->] (0) edge[bend right = 20] (3);
	\draw[->] (0) -- (4);
	\draw[->] (0) edge[bend left = 10] (5);
	\draw[->] (2) -- (6);
	\draw[->] (3) -- (6);
	\draw[->] (4) edge[bend left = 10] (6);
	\draw[->] (5) -- (6);
	\draw[->] (1) edge[bend right = 20] (6);
	\node[] at (0,-2) {(A)};
	\end{scope}
	\end{tikzpicture}
	\qquad
	\begin{tikzpicture}
	[>=stealth, rv/.style={circle, draw, thick, minimum size= 5mm}, rvc/.style={triangle, draw, thick, minimum size=5mm}, node distance=20mm]
	\pgfsetarrows{latex-latex};
	\begin{scope}
	\node[rv] (0) at (-2,0) {$A$};
	\node[rv] (1) at (-1,1) {$X_2$};
	\node[rv] (2) at (0,-1) {$X_1$};
	\node[rv] (3) at (1,1) {$X_3$};
	\node[rv] (6) at (2,0) {$Y$};
	\draw[->] (0) -- (1);
	\draw[->] (0) -- (2);
	\draw[->] (1) -- (3);
	\draw[->] (2) -- (6);
	\draw[->] (3) --  (6);
	\draw[->] (1) edge[bend right = 20] (6);
	\node[] at (0,-2) {(B)};
	\end{scope}
	\end{tikzpicture}
	\caption{A graphical model representation of the SEMs used for Synthetic examples A and B.}
	\label{fig:gapcalib}
\end{figure}
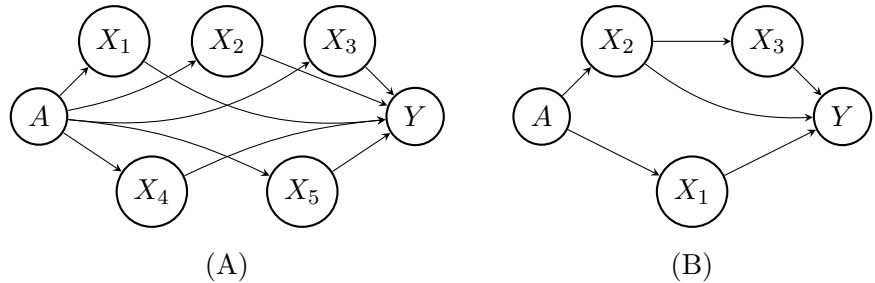
In both examples, we analyse the AUC-parity gap and parity gap-calibration score trade-offs via the resolving variables. Two baseline methods were implemented, to compare our results. These are \textit{reweighing} \citep{kamiran2012data} and \textit{fair reductions} \citep{agarwal2018reductions}. Both of these methods aim to achieve demographic parity. Therefore, we should compare them to our method with no resolving variables. More details on comparison methods are given shortly in Section \ref{realdata}. The fair reductions approach performs poorly on both of these task (for a range of parameter values $\epsilon$ of the method), so it is not included in the final analysis of the results.

In example A we enlarge the set of resolving variables stepwise, including $X_i$ at step $i$. For example B, we try out all possible subsets of resolving variables. We run our method ten times, with 5000 training and test samples generated from the given SEMs. A logistic regression classifier is used after applying \texttt{fairadapt}. On each repeat we measure the AUC, parity gap and the calibration score. The results are shown in Figures \ref{fig:ExampleA} and \ref{fig:ExampleB}. Vertical error bars in the figures represent the standard deviations of respective measures obtained from the ten repeats.
\begin{figure}
    \centering
    \begin{minipage}{0.45\textwidth}
        \centering
        \includegraphics[width=1.1\textwidth]{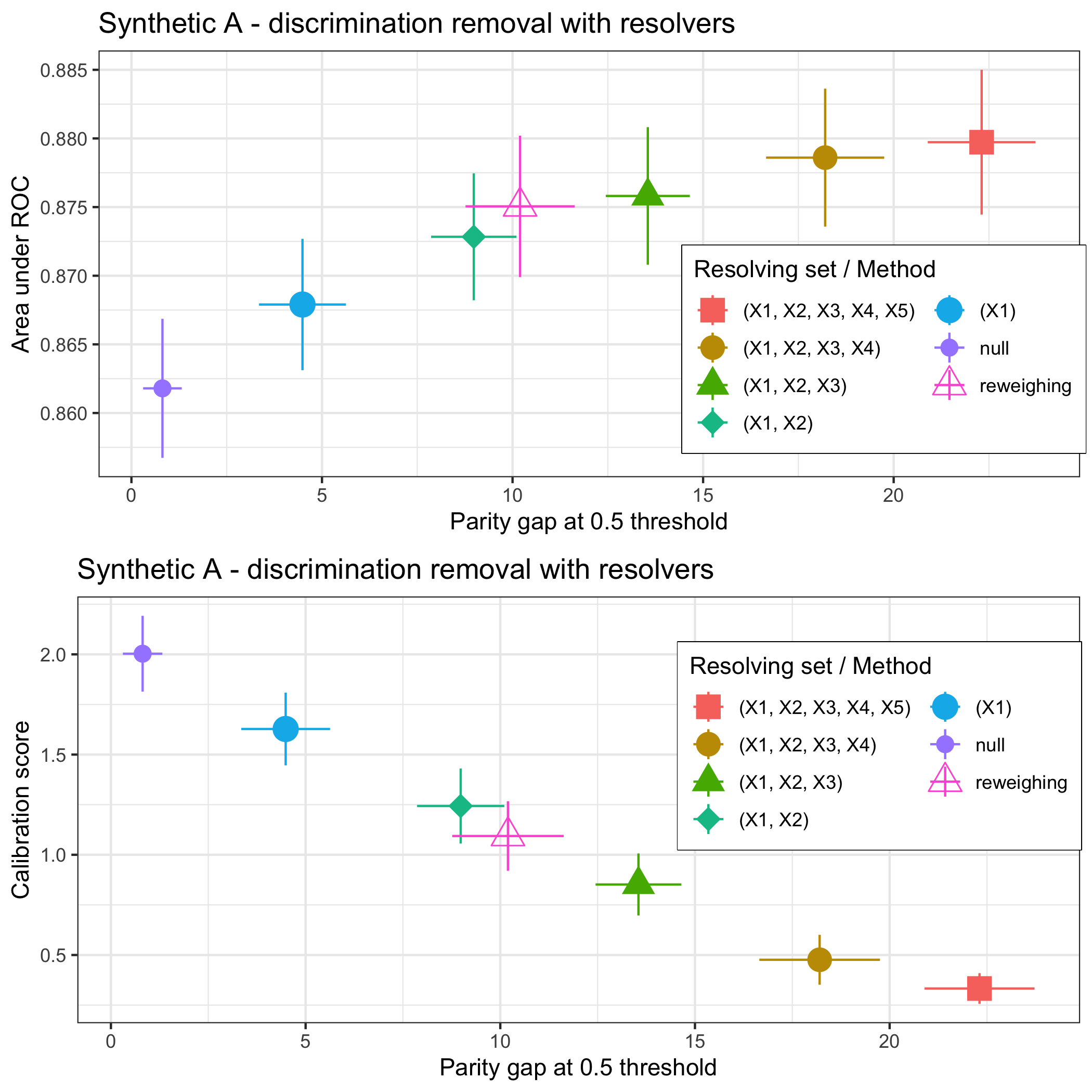} 
        \caption{AUC-parity and parity-calibration score trade-off for example A. Vertical bars represent standard deviations obtained from 10 repeats.}
				\label{fig:ExampleA}
    \end{minipage}\hfill
    \begin{minipage}{0.45\textwidth}
        \centering
        \includegraphics[width=1.1\textwidth]{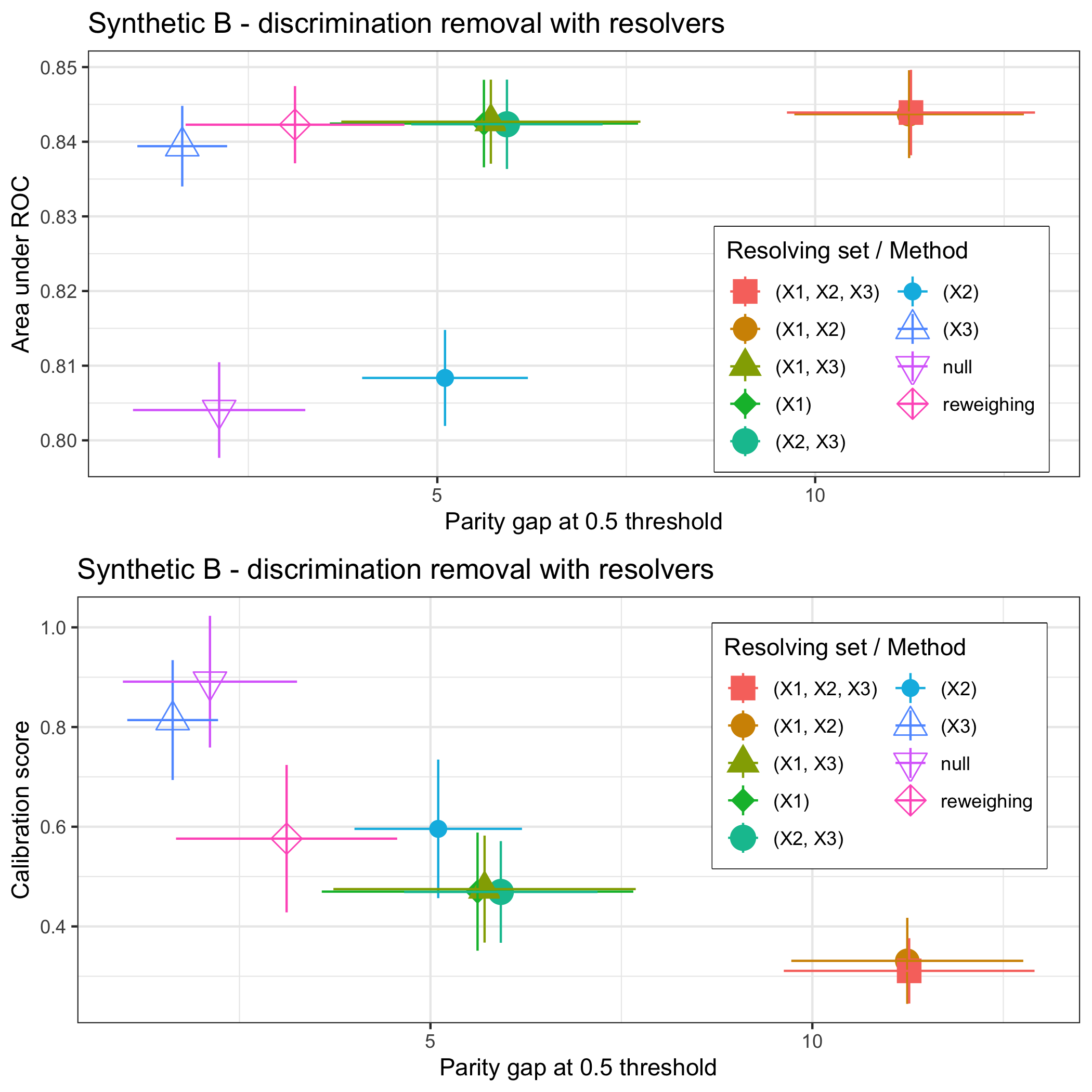} 
        \caption{AUC-parity and parity-calibration score trade-off for example B. Vertical bars represent standard deviations obtained from 10 repeats.}
				\label{fig:ExampleB}
    \end{minipage}
\end{figure}

\paragraph{Example A.} In Figure \ref{fig:ExampleA} the AUC and the parity gap are increasing, whereas the calibration score is becoming smaller, as we enlarge the resolving set. Here we can see the trade-off between demographic parity and calibration via the resolving variables. The baseline method of reweighing obtains better accuracy, but fails to eliminate discrimination fully.

\paragraph{Example B.} Figure \ref{fig:ExampleB} shows similar behaviour, in a slightly more complicated causal situation. For example, note that setting $X_1$ to be resolving implicitly sets $X_3$ to be resolving, too. Setting $R = \lbrace X_1, X_2 \rbrace$ has almost the same performance as setting $R = \lbrace X_1, X_2, X_3 \rbrace$. Similary, resolving sets $\lbrace X_2, X_3 \rbrace$ and $\lbrace X_2 \rbrace$ show very similar results, as expected.

\subsection{Real data experiments} \label{realdata}
We next look at real data experiments. We summarize all the baseline methods against which we benchmark our results. In the real data comparisons, we only consider the case of demographic parity (meaning no resolving variables) as other comparisons methods are designed to achieve precisely this notion.

\paragraph{Baseline methods.}
The comparison methods that we look at are:
\begin{itemize}
	 \item standard implementaion of random forests \citep{wright2015ranger}, serving as a fairness-ignorant baseline
	 \item \textit{fairness through unawareness} - RF applied to the data after excluding the protected attribute
	 \item the \textit{reweighing} preprocessing method \citep{kamiran2012data}, which learns specific weights for the combinations of the class label and the protected attribute which are then used for building a classifier (in this case we use the logistic regression classifier and the implementation from the IBM toolkit \citep{ibmtoolkit})
	 \item the \textit{reductions} approach \citep{agarwal2018reductions} casts the fairness problem in a linear programming (LP) form in order to find a sample-weighted classifier which satisfies the desired fairness constraint (we again use logistic regression for our classifier that allows sample-weighting and vary the fairness constraint violation parameter $\epsilon \in \lbrace 0.1, 0.01, 0.001 \rbrace$)
\end{itemize}
\paragraph*{UCI Adult.}
The Adult dataset from the UCI machine learning repository \citep{lichman2013uci} contains information on 48842 individuals and the outcome to be predicted is whether an individual has a yearly income of more than 50 thousand dollars. The data comprises of the following features\footnote{The original dataset contains a few more features, but we focus on those that have been used in previous fairness applications.}:
\begin{itemize}
	\item gender, labelled $A$, which we consider to be the protected attribute
	\item demographic information $C$ - including age, race and nationality
	\item marital status $M$ and years of education $L$
	\item work related information $R$ - job occupation, hours of work per week and work class
	\item a binary outcome  $Y$ representing whether a person's income exceeds 50000 dollars a year
\end{itemize}
The UCI Adult dataset has been previously analyzed as an application of different fairness procedures, for instance in \citep{shpitser2018} and \citep{pathspecific}. The proposed causal graph for the dataset is presented in Figure \ref{fig:causalgraphs}(a).
\begin{figure} \centering
	\begin{tikzpicture}
	[>=stealth, rv/.style={circle, draw, thick, minimum size=7mm}, rvc/.style={triangle, draw, thick, minimum size=8mm}, node distance=7mm]
	\pgfsetarrows{latex-latex};
	\begin{scope}
	\node[rv] (c) at (2,2) {$\pmb{C}$};
	\node[rv] (a) at (-2,2) {$A$};
	\node[rv] (m) at (-3,0) {$M$};
	\node[rv] (l) at (-1,0) {$L$};
	\node[rv] (r) at (1,0) {${R}$};
	\node[rv] (y) at (3,0) {$Y$};
	\draw[->] (c) -- (m);
	\draw[->] (c) -- (l);
	\draw[->] (c) -- (r);
	\draw[->] (c) -- (y);
	\draw[->] (a) -- (m);
	\draw[->] (m) -- (l);
	\draw[->] (l) -- (r);
	\draw[->] (r) -- (y);
	\path[->] (a) edge[bend left = 0] (l);
	\path[->] (a) edge[bend left = 0] (r);
	\path[->] (a) edge[bend left = 0] (y);
	\path[->] (m) edge[bend right = 20] (r);
	\path[->] (m) edge[bend right = 30] (y);
	\path[->] (r) edge[bend right = 20] (y);
	\path[->,red,dashed] (a) edge[bend left = 10] node[above left] {(?)} (c);
	\node[] at (0,-2) {(a)};
	\end{scope}
	\end{tikzpicture}
	\qquad
	\begin{tikzpicture}
	[>=stealth, rv/.style={circle, draw, thick, minimum size=7mm}, rvc/.style={triangle, draw, thick, minimum size=8mm}, node distance=7mm]
	\pgfsetarrows{latex-latex};
	\begin{scope}
	\node[rv] (c) at (2,2) {$\pmb{C}$};
	\node[rv] (a) at (-2,2) {$A$};
	\node[rv] (m) at (-3,0) {$\pmb{J}$};
	\node[rv] (l) at (-1,0) {$P$};
	\node[rv] (r) at (1,0) {${D}$};
	\node[rv] (y) at (3,0) {$Y$};
	\draw[->] (c) -- (m);
	\draw[->] (c) -- (l);
	\draw[->] (c) -- (r);
	\draw[->] (c) -- (y);
	\draw[->] (a) -- (m);
	\draw[->] (m) -- (l);
	\draw[->] (l) -- (r);
	\draw[->] (r) -- (y);
	\path[->] (a) edge[bend left = 0] (l);
	\path[->] (a) edge[bend left = 0] (r);
	\path[->] (a) edge[bend left = 0] (y);
	\path[->] (m) edge[bend right = 20] (r);
	\path[->] (m) edge[bend right = 30] (y);
	\path[->] (r) edge[bend right = 20] (y);
	\node[] at (0,-2) {(b)};
	\end{scope}
	\end{tikzpicture}
	\caption{(a) the causal graph (black edges) claimed to correspond to the UCI Adult dataset. The additional red, dashed edge corresponds to a sampling bias in the data; (b) causal graph of the COMPAS dataset.}
	\label{fig:causalgraphs}
\end{figure}
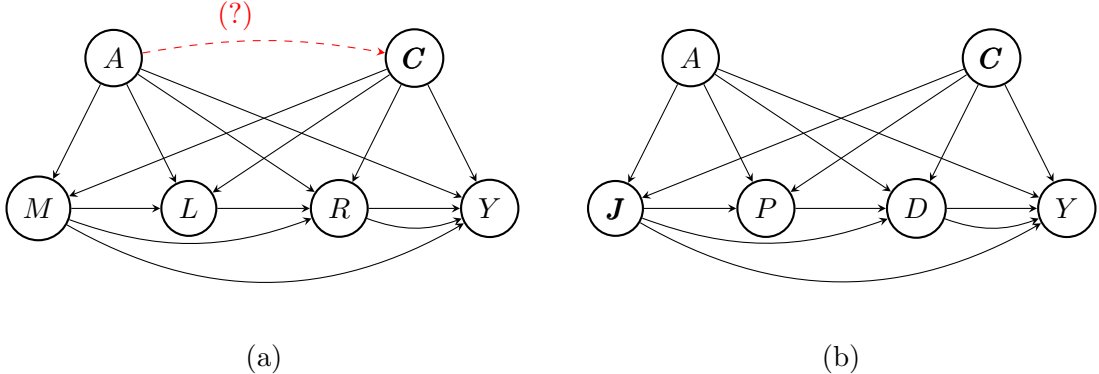
While we do agree that this causal graph makes sense intuitively, care needs to be taken because the sampling bias can induce dependencies that have no explanation in reality. This is precisely the case with the UCI Adult dataset, which we can observe by inspecting the relation between two features in $\pmb{C}$ (age and race) and the protected attribute $A$. From the plots in Figure \ref{fig:UCIplots} we see that in the dataset gender is not independent of age and race, as the causal graph would imply. To solve the problem, we subsample the dataset in order to mitigate the sampling bias. Details about how we pre-processed the dataset are given in Appendix A.
\begin{figure}
	\centering
	\includegraphics[height=60mm,angle=0]{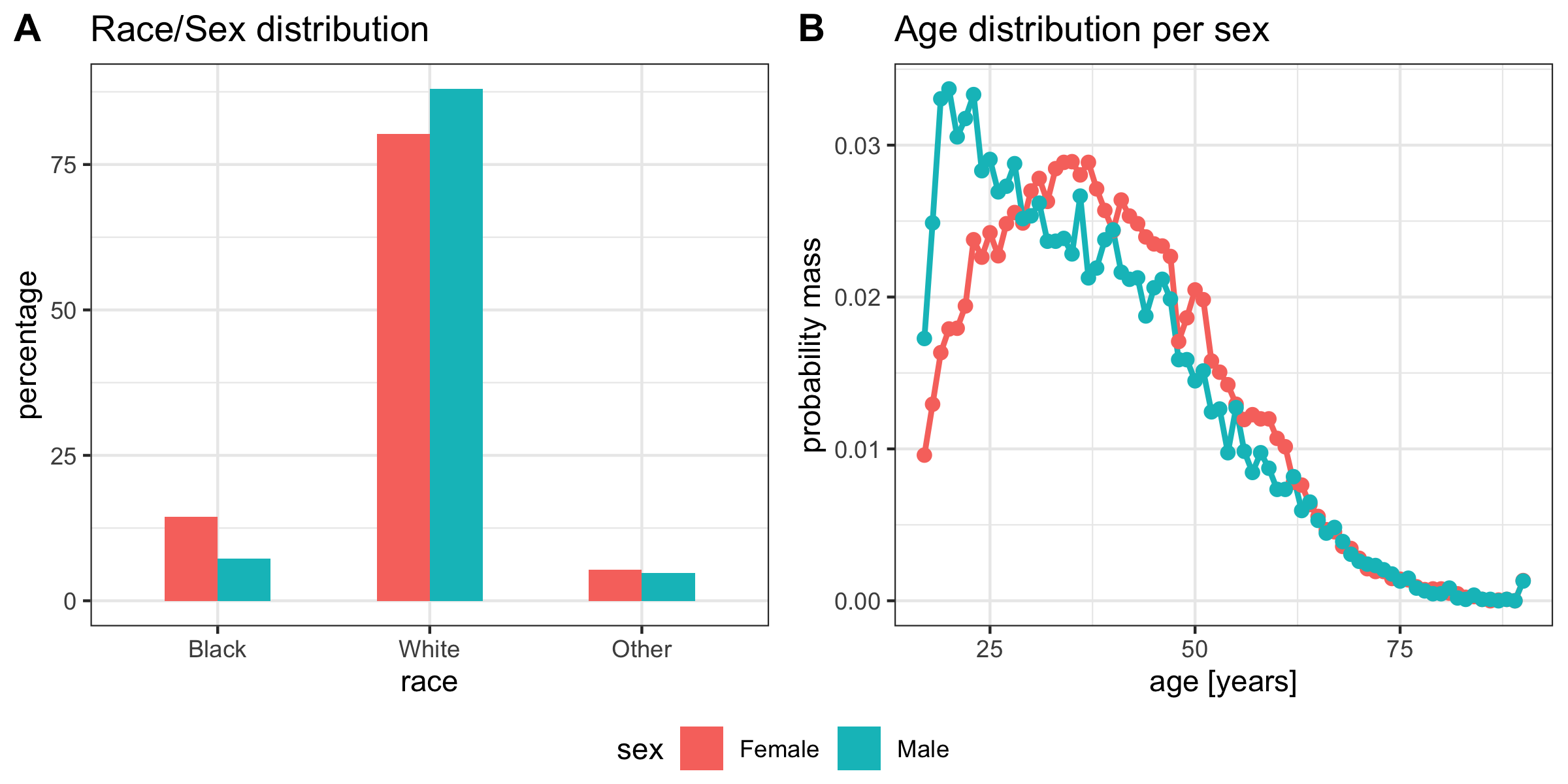}
	\caption{Plots of influence of sex on race (\textbf{A}) and age (\textbf{B}) in the UCI Adult dataset.}
	\label{fig:UCIplots}
\end{figure}

\paragraph*{COMPAS dataset.} The second real dataset we analyse is the COMPAS dataset \citep{ProPublica} which contains the following features:
\begin{itemize}
	\item outcome $Y$ is recidivism whilst on parole within a two year period
	\item protected attribute $A$ in this case is race (White vs. Non-White)
	\item demographic information $\pmb{C}$
	\item juvenile offense counts $\pmb{J}$, count of prior offenses $P$ and degree of the charge $D$
\end{itemize}

\noindent The causal graph that we propose is given in Figure \ref{fig:causalgraphs}(b). The reader here might disagree that this example falls into the class of Markovian non-parametric models.

We select several individuals in the dataset which are non-white, male and of age 30. We look at their values of juvenile counts and prior counts $(J_1, J_2, J_3, P)$ before and after applying \texttt{fairadapt}
\begin{verbatim}
> compas[id, offense.count]
          juv_fel_count juv_misd_count juv_other_count priors_count
241              0              0               0            4
646              0              0               0            8
807              0              0               0           17
1425             2              0               0           20
1470             1              0               2           15
> adapted_compas[id, offense.count]
          juv_fel_count juv_misd_count juv_other_count priors_count
241              0              0               0            3
646              0              0               0            5
807              0              0               0           13
1425             0              0               0           11
1470             0              0               2            9
\end{verbatim}
We can notice that fair adaptation reduces the number of offenses for these individuals, since in the dataset the baseline population (white) has fewer offenses on average.  Notice how the transformed values could be used in an interpretable way. Hypothetical statements like "if you were white, your juvenile offense counts would have been $J_1, J_2, J_3$, in turn resulting in prior count of $P$, resulting in prediction $\widehat{Y}$" now become possible. This part of our method, however, rests on the assumption from Definition \ref{def:qpa}.
\paragraph*{Results. } For both UCI Adult and COMPAS, we split the dataset into 75\% training and 25\% testing randomly 20 times. Each time, we apply all the baseline methods and our \texttt{fairadapt} method, measuring accuracy and the parity gap each classifier achieves. Figures \ref{fig:adult} and \ref{fig:compas} summarize the obtained results.
\begin{figure}
    \centering
    \begin{minipage}{0.45\textwidth}
        \centering
        \includegraphics[width=1.1\textwidth]{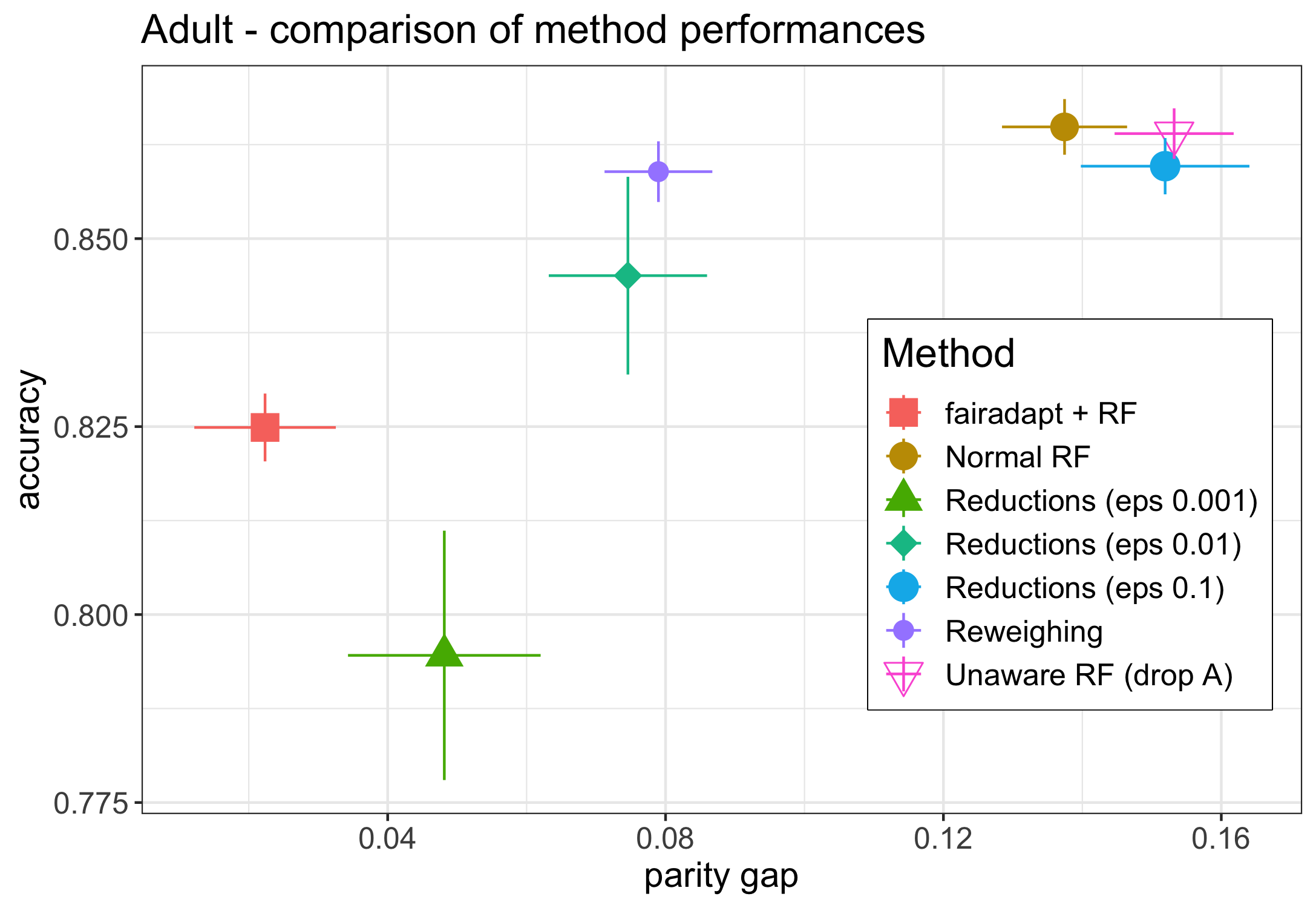} 
        \caption{Comparison of the performance of different fairness methods on the UCI Adult dataset. Vertical bars represent standard deviations obtained from 20 repeats.}
				\label{fig:adult}
    \end{minipage}\hfill
    \begin{minipage}{0.45\textwidth}
        \centering
        \includegraphics[width=1.1\textwidth]{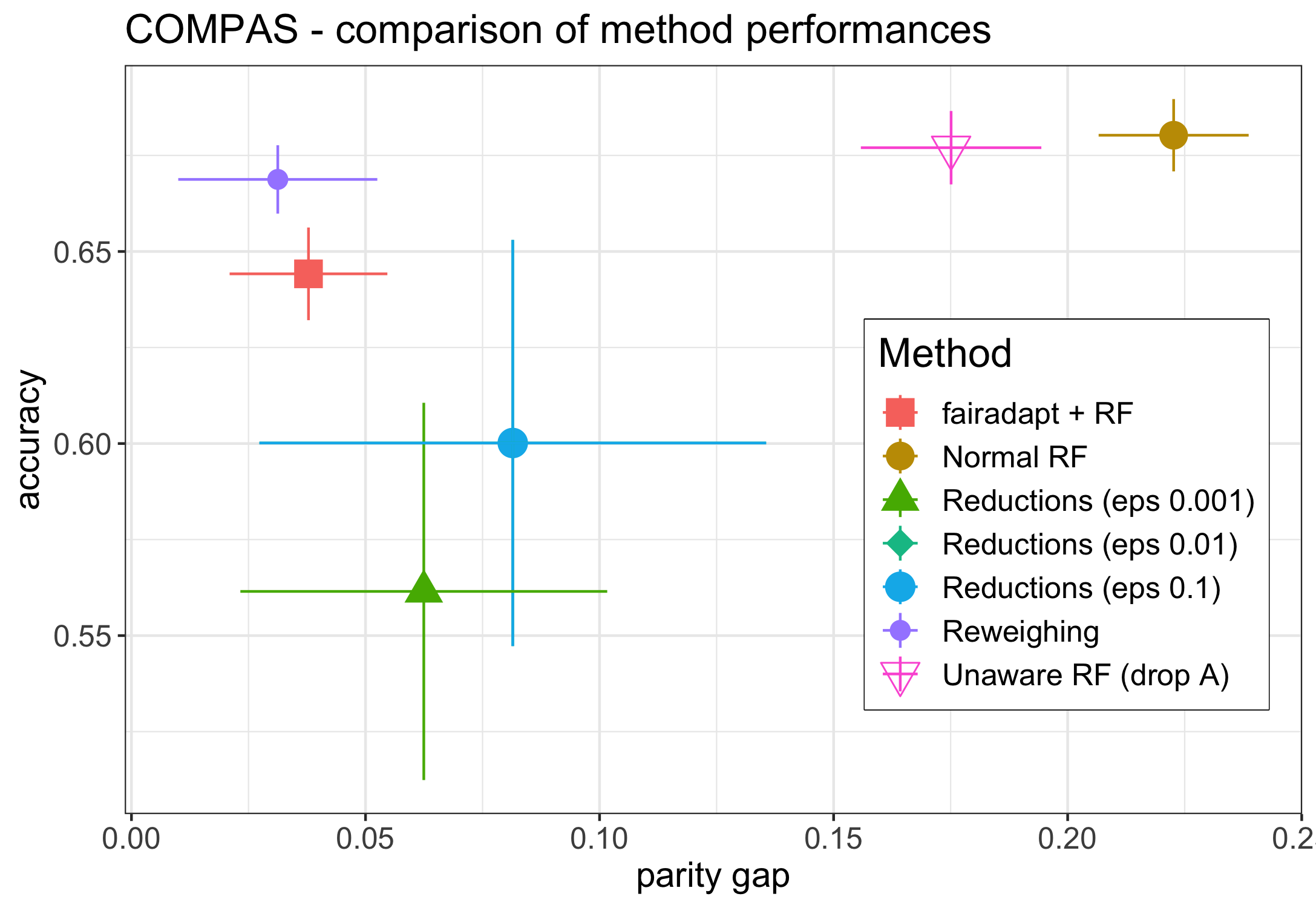} 
        \caption{Comparison of the performance of different fairness methods on the COMPAS dataset. Vertical bars represent standard deviations obtained from 20 repeats.}
				\label{fig:compas}
    \end{minipage}
\end{figure}
For the Adult dataset, no method is better than \texttt{fairadapt} on both criteria. For the COMPAS dataset, only the reweighing method is marginally better. We note our method has very satisfying performance. On top of this, we mention that our method has the ability to relax the fairness criterion via resolving variables, has a causal interpretation and allows for individual level interpretability (under the QPA).

Finally, we take a look at how \texttt{fairadapt} affects the distribution of the positive outcome probabilities. We plot the densities of $\pr(\widehat{Y} = 1 \mid A = a)$ for both levels of $A$ for the two cases of not applying and applying \texttt{fairadapt}. The results are shown in Figures \ref{fig:adultdensity} and \ref{fig:compasdensity}. Note that the densities are much closer when applying \texttt{fairadapt}, indicating a clear reduction in discrimination.

\begin{figure}
    \centering
    \begin{minipage}{0.45\textwidth}
        \centering
        \includegraphics[width=1.1\textwidth]{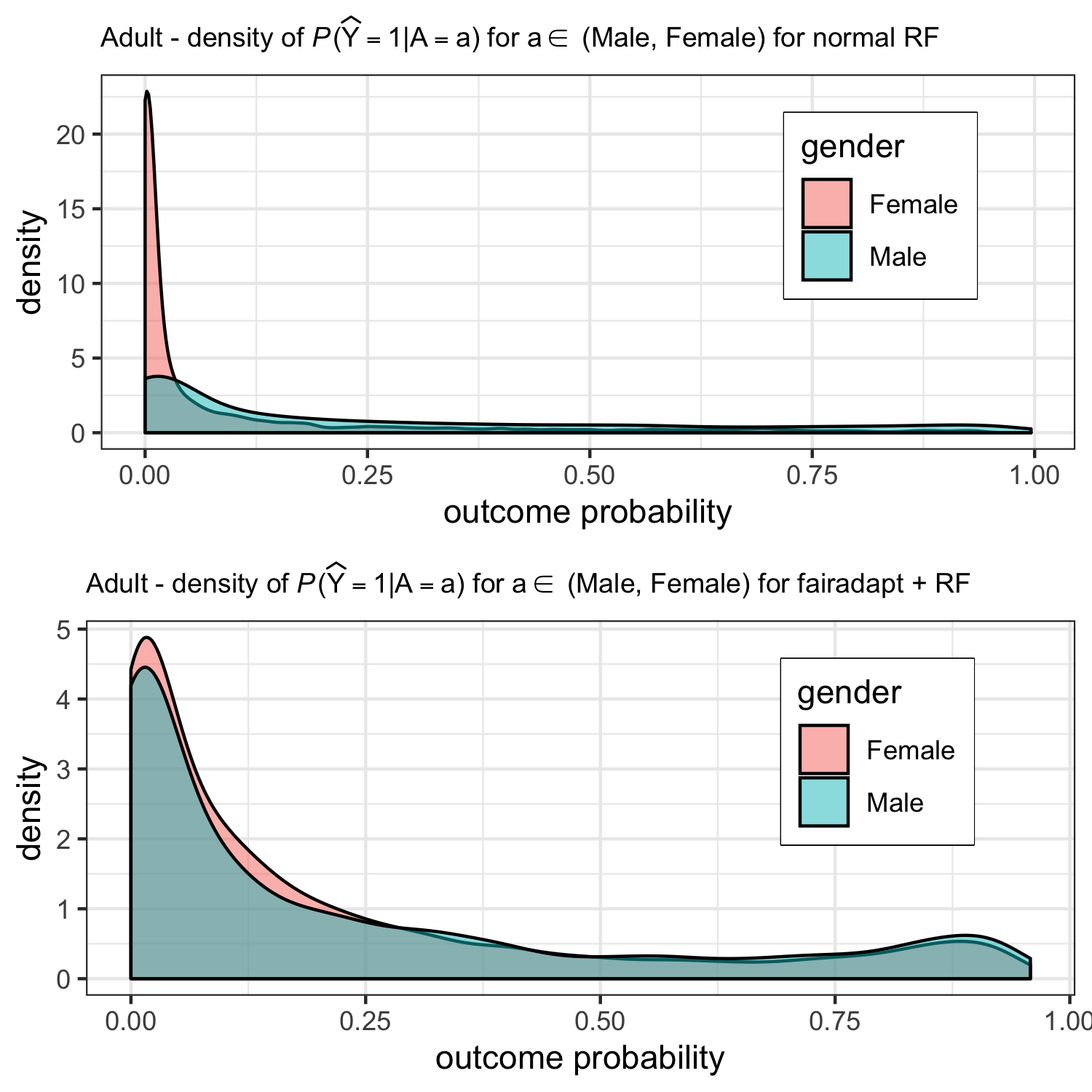} 
        \caption{Change in the positive outcome probability density due to applying \texttt{fairadapt} to UCI Adult.}
				\label{fig:adultdensity}
    \end{minipage}\hfill
    \begin{minipage}{0.45\textwidth}
        \centering
        \includegraphics[width=1.1\textwidth]{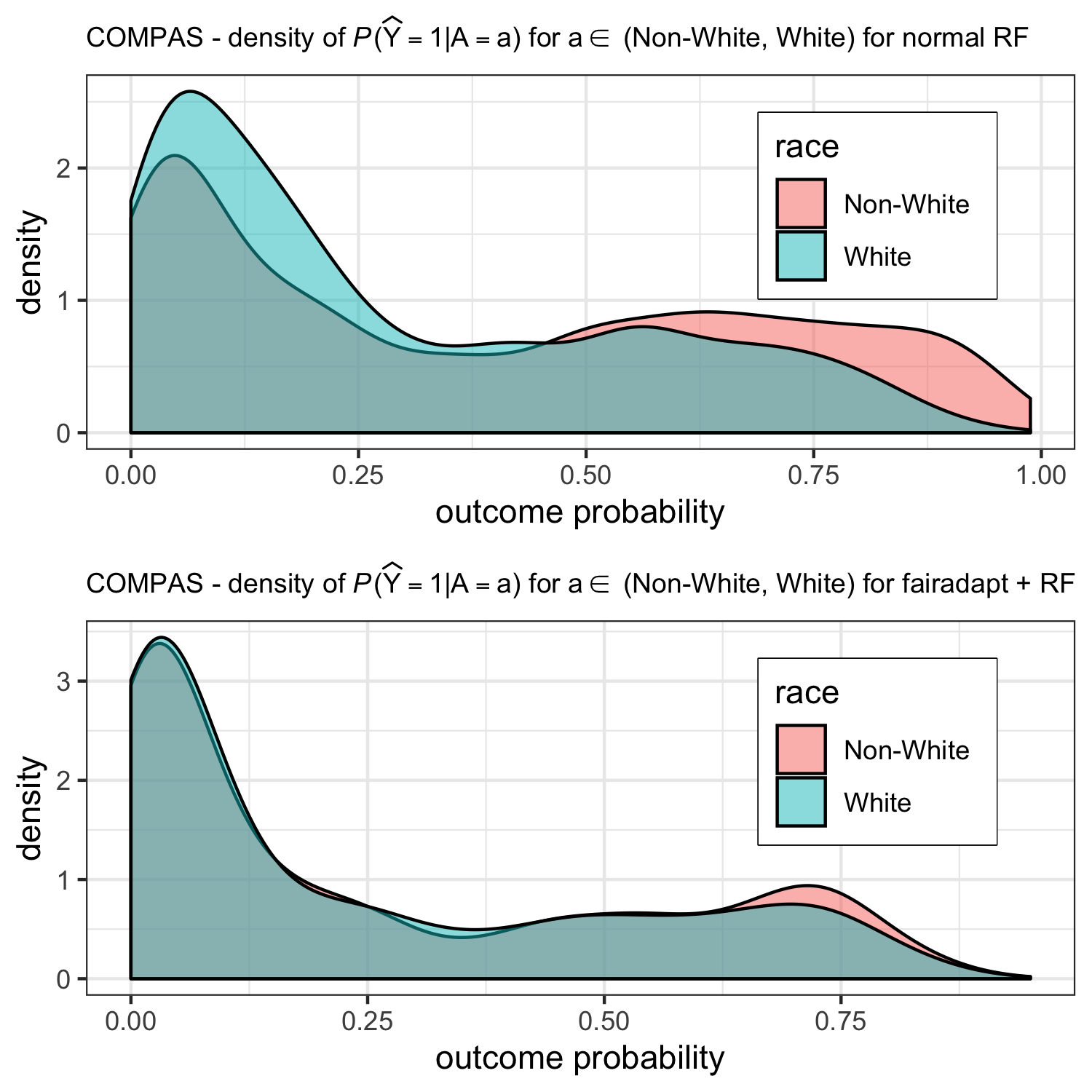} 
        \caption{Change in the positive outcome probability density due to applying \texttt{fairadapt} to COMPAS.}
				\label{fig:compasdensity}
    \end{minipage}
\end{figure}

\section{Conclusion} \label{Conclusion}
In the final section we revisit some of the ideas discussed previously and conclude our argument.
\paragraph*{About observational criteria.} Causal and observational notions of fairness have an inherent link. If the protected attribute is a root node, the intervention on $A$ is equivalent to conditioning on $A$. Causality is necessary, not just to provide new criteria, but to give meaning to the existing observational criteria used.

\paragraph*{About fair data adaptation.} We conclude that \texttt{fairadapt} shows competitive performance compared to other baseline methods in the case of demographic parity. It also gives a causal and interpretable perspective on the data transformation that is carried out. Further, it offers various relaxations of demographic parity, all the way to the case of calibration, which is achieved when all the variables are considered to be resolving. The output of fair data adaptation also allows us to see which values individuals were assigned in the transformation procedure. This helps justify and interpret why a certain individual was given his prediction.

\paragraph*{About the current datasets and methods.} We emphasize that it would be beneficial for the advancement of fairness if there were established real world datasets with agreed-upon causal graphs. This would allow different authors to compare their methods in a meaningful way, demonstrating the performance and measuring different fairness criteria. Having a benchmark for algorithm performance and fairness criteria achieved could also help us understand how and why different methods yield different results on the same datasets. We feel like this is not yet the case and that much more could be done on this front.

\paragraph*{About future work.} We have discussed a method which achieves certain fairness criteria and shown how it can be used in practice. However, this is only the very first step of fairness. A big component of the whole problem that has so far been barely discussed is the temporal implications of fairness criteria on well-being of different groups. The only work on this topic we are currently aware of is \citep{hardtdelayed}. Although many of the fairness criteria make intuitive sense and perhaps have some philosophical backing, we have no reason to convince ourselves that they are necessarily doing the right thing in terms of their long-term effect. This is a serious question and perhaps more involvement is needed from the economics community - many already developed tools could be very useful in this discussion.


\paragraph{Acknowledgements:} We would like to thank Domagoj {\'C}evid, Yuansi Chen and Federico Glaudo for useful discussions and suggestions that greatly improved our work.


\newpage

\appendix
\section{Proof of Theorem \ref{thm:populationlevel}} \label{appendix:theoremproof}
\begin{proof}
	We prove that the transformation $FT(\cdot)$ satisfies $$ FT(X(U = u)) \quad = \quad X(A = 0, R = r, U = u) \quad \forall r, u.$$
	Take any $U = u$ such that $R(U = u) = r$. Under the do$(A = 0, R = r)$ intervention the assignment equations of $A$ and $R$ change to
	\begin{align*}
	A &\gets 0, \\
	R &\gets r,
	\end{align*}
meaning that $FT(A(U = u)) = A(A = 0, U = u)$ and $FT(R(U = u)) = R(A = 0, R = r, U = u)$. Also, for any $V$ non-descendant of $A$ or $R$ we have that
$$ FT(V(U = u)) = V(A = 0, R = r, U = u).$$
We proceed inductively. Let $U^{(V)}$ be the component of $U = u$ corresponding to variable $V$. In the first step, for any $V \in \text{ch}(A)$ we can show that
	\begin{align}
		V(A = 0, R = r, U = u) &= g(\text{pa}(V)(A = 0, R = r, U = u), U^{(V)}) \label{qpr}\\
		&= g(FT(\text{pa}(V)), U^{(V)}) \nonumber\\
		&= FT(V(U = u)) \nonumber
	\end{align}
where the first equality holds by definition of the intervention and the quantile preservation assumption (QPA), the second because we showed $FT(V(U = u)) = V(A = 0, R = r, U = u)$ for all $V \in \lbrace A, R , \text{nde}(A) \rbrace$ (here nde$(A)$ are non-descendants of $A$), the third from the definition of Algorithm \ref{algo:fairnesspopulation}. Using the fact that Algorithm \ref{algo:fairnesspopulation} goes through variables $V$ in topological order, inductively we can show $FT(V(U = u)) = V(A = 0, R = r, U = u)$ for any $V$ in ch(ch($A$)) and so on. This shows that strong resolved fairness holds under the QPA. If the QPA is not used, then the equality \eqref{qpr} does not hold anymore. However, even without QPA it still holds that $V(A = 0, R = r, U') \overset{d}{=} g(\text{pa}(V)(A = 0, R = r, U), U^{(V)})$ where $U, U'$ are now viewed as independent random variables (with a $U[0,1]^{p+2}$ distribution). This is enough to guarantee that $X(A = a, R = r) \; \overset{d}{=} \; X(A = a', R = r) \; \forall a,a',r$. From this it follows that for any classifier $\widehat{Y} = f \circ FT$ we have that $$\widehat{Y}(A = a, R = r) \overset{d}{=} \widehat{Y}(A = a', R =r ).$$
\end{proof}

\section{Probability predictions satisfying resolver-induced parity gap} \label{appendix:whyprobability}
Take the following simple example
\begin{gather*}
	A \gets \text{Bernoulli}(0.5) \\
	X_1 \gets \frac{1}{2}\mathbb{1}(A = 0)+ \epsilon_1 \\
	X_2 \gets \frac{2}{3}(\mathbb{1}(A = 0)-\frac{1}{2}) + \epsilon_2 \\
	Y \gets \text{Bernoulli}(\text{expit}(X_1+X_2))
\end{gather*}
where $\epsilon_1, \epsilon_2$ are both $N(0, \sigma^2)$ variables with $\sigma^2 = 0.05$. Variable $A$ represents gender, with $A = 0$ being the male population. Suppose that $X_2$ is resolving and $X_1$ is not. After adaptation (assuming no estimation error) we have that $FT(X_1) \gets \epsilon_1$ and $FT(Y) \gets \text{Bernoulli}(\text{expit}(FT(X_1)+X_2))$. Plot of the density of the probability of a positive outcome $\pr(FT(Y) = 1 \mid A = a)$ are shown in Figure \ref{fig:RIPGplot}. Note that an optimal probability predictor $\widehat{Y} = \ex \big[ FT(Y) \mid FT(X) \big]$ would have
\begin{align*}
	\ex \big[ \widehat{Y}(A = 0) - \widehat{Y}(A = 1) \big] = \ex \big[ FT(Y) \mid A = 0\big] - \ex \big[ FT(Y) \mid A = 1\big] \approx 0.164.
\end{align*}
However, an optimal $\lbrace 0,1 \rbrace$ classifier $\widetilde{Y}$ trying to minimize (for example) the $L_2$-loss would simply be constructed as $\widetilde{Y} = \mathbb{1}(\widehat{Y} \geq \frac{1}{2})$. Note that (referring to Figure \ref{fig:RIPGplot}) for this $\widetilde{Y}$ we have that $$\ex \big[ \widetilde{Y}(A = 0) - \widetilde{Y}(A = 1)\big] \approx 1.$$
\begin{figure}
	\centering
	\includegraphics[height=60mm,angle=0]{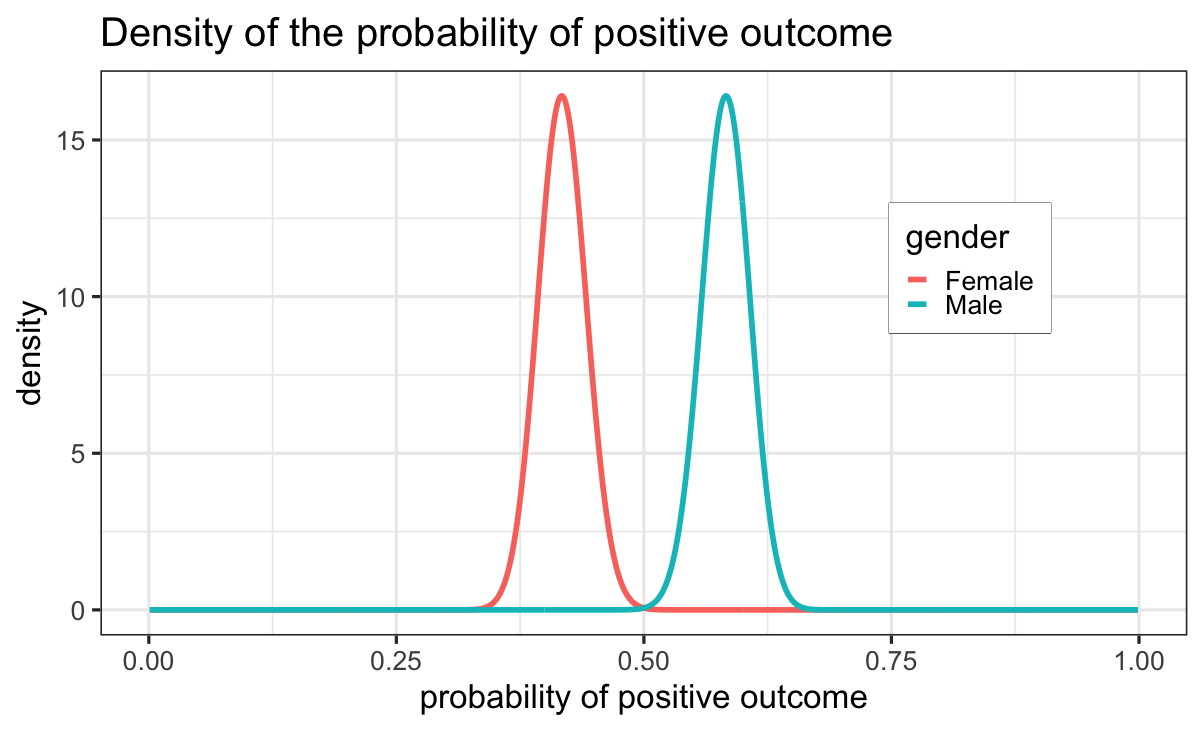}
	\caption{Density of the probability of positive outcome $\pr(FT(Y) = 1)$.}
	\label{fig:RIPGplot}
\end{figure}
Due to examples like this, the criterion \eqref{eq:RIPG} was defined for probability and not class predictions. A much more involved, general discussion of this problem is given in Section \ref{methodformalisation}.
\section{Edge specific extension} \label{appendix:edgeextension}
Consider a dataset consisting of the following features\footnote{This example, not surprisingly, is motivated by COMPAS.}:
\begin{itemize}
	\item protected attribute $A$, in this case race
	\item information about amount of policing the person experiences, $P$ (given explicitly or perhaps implicitly through a ZIP code)
	\item information about prior convictions $C$
	\item recidivism outcome $Y$ when the person is released on parole
\end{itemize}
A possible causal graph for this dataset is given in Figure \ref{fig:policing}(a).
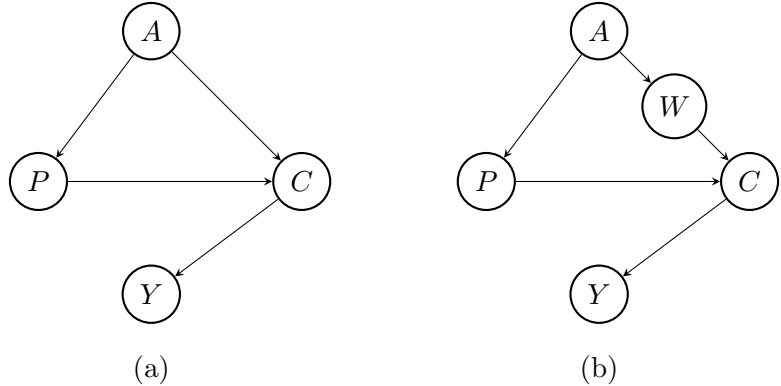
\begin{figure} \centering
	\begin{tikzpicture}
	[>=stealth, rv/.style={circle, draw, thick, minimum size=6mm}, rvc/.style={triangle, draw, thick, minimum size=7mm}, node distance=18mm]
	\pgfsetarrows{latex-latex};
	\begin{scope}
	\node[rv] (1) at (0,2) {$A$};
	\node[rv] (2) at (-1.5,0) {$P$};
	\node[rv] (3) at (2,0) {$C$};
	\node[rv] (4) at (0,-1.5) {$Y$};
	\draw[->] (1) -- (2);
	\draw[->] (1) -- (3);
	\draw[->] (2) -- (3);
	\draw[->] (3) -- (4);
	\end{scope}
	\node[] at (0,-2.5) {(a)};
	\end{tikzpicture}
	\qquad
	\qquad
	\begin{tikzpicture}
	[>=stealth, rv/.style={circle, draw, thick, minimum size=6mm}, rvc/.style={triangle, draw, thick, minimum size=7mm}, node distance=18mm]
	\pgfsetarrows{latex-latex};
	\begin{scope}
	\node[rv] (1) at (0,2) {$A$};
	\node[rv] (2) at (-1.5,0) {$P$};
	\node[rv] (5) at (1,1) {$W$};
	\node[rv] (3) at (2,0) {$C$};
	\node[rv] (4) at (0,-1.5) {$Y$};
	\draw[->] (1) -- (2);
	\draw[->] (1) -- (5);
	\draw[->] (2) -- (3);
	\draw[->] (3) -- (4);
	\draw[->] (5) -- (3);
	\end{scope}
	\node[] at (0,-2.5) {(b)};
	\end{tikzpicture}
	\caption{(a) example that motivates the edge extension of the idea of resolving variables; (b) example where an edge specific extension might arise naturally.}	\label{fig:policing}
\end{figure}
\cite{shpitser2018} considered the variable $C$ as resolving in the COMPAS dataset. If, however, information about policing is available, we might want to account for this. Suppose that the difference in prior convictions between the black and white population was partly due to the fact that black people experience more policing. We would, in this case, consider this effect unfair. Therefore, we need to find a way to remove the $A \rightarrow P \rightarrow C$ effect, but keep the direct $A \rightarrow C$ effect. This example demonstrates that sometimes we perhaps want to have \textit{partially resolving variables}.

We argue that sometimes it is hard to choose if a variable is simply resolving or non-resolving. Going back to the case of policing from Figure \ref{fig:policing}, it would be difficult to determine whether the prior convictions variable $C$ is resolving or non-resolving. In some sense, both choices would be wrong. We therefore think the approach of choosing which edges to remove allows for some additional flexibility with modelling. Another way in which the edge extension might arise naturally is the following. Imagine that the path $A \rightarrow C$ was actually going through some unmeasured variable $W$, as shown in Figure \ref{fig:policing}(b). If we considered $W$ as resolving, removing the effect $A \rightarrow C$ in the original graph might capture what we want to achieve with our adaptation.

For every variable $V$ we need to define its \textit{adaptation parent set}, $\aps (V) \subset \pa (V)$, determining which of the parent variables must change when computing the counterfactual value. The adaptation parent set $\aps(R)$ is the subset of parents whose unfair effect we wish to remove. For a resolving variable $R$, $\aps(R) = \emptyset$. For a non-resolving variable $X$ we have that $\aps(X) = \pa(X)$. In the example from Figure \ref{fig:policing}(c) we have $\aps(C) = {P}$.

The main difference from the original version is in line \ref{CFassign} of Algorithm \ref{algo:fairnesspopulation}, in which we assign the transformed value as
\begin{equation} \label{eq:assign}
	FT(V_k) \gets g_V(U_k,\ FT(\pa(V_k))).
\end{equation}
In the edge specific case, instead of using transformed values of all the parents $\pa(V)$ in the assignment \eqref{eq:assign}, we use the original values of parents in $\pa(V) \setminus \aps(V)$ and the transformed values $FT(\aps(V))$ of the parents in $\aps(V)$.

\section{UCI Adult dataset} \label{appendix:UCI}
We give more details about how we preprocessed the UCI Adult dataset. The preliminary cleaning of the dataset is similar to that of \cite{cleanadult}. In particular, the following operations on the features are performed:
\begin{itemize}
	\item variables "relationship", "final weight", "education" (categorical), "capital gain" and "capital loss" were removed
	\item levels of variable "work class" were merged, so that we obtain four different levels - Government, Self-Employed, Private and Other/Unknown
	\item levels of the variable "marital status" were merged so that we obtain two levels - Married and Not-Married
	\item levels of variable "native country" were merged so that we obtain two levels - US and Non-US
\end{itemize}
Categorical variables that are descendants of gender $A$ were given an ordering, so that the probability of success $\pr(Y = 1 \mid F = f)$ is marginally increasing in levels of $F$. This is described more precisely in Section \ref{categorical}.

From Figure \ref{fig:UCIplots} we see that females in the dataset are much more likely to be in their early twenties and are also more likely to be black than males. Since we do believe that additional edges between $A$ and $C$ are present only due to sampling, we propose a subsampling method to resolve the problem and obtain a dataset for which the causal graph in Figure \ref{fig:causalgraphs}(a) is valid. In particular, we take only the white subpopulation. Since there are strictly more males than females for every age value, we subsample the males randomly so that we achieve exact matching in the age distributions between genders. In this way, we avoid the problem of biased sampling. The dataset still consists of 26052 individuals, which is a sufficient amount of data.

\section{Quantiles are fair} \label{appendix:quantiles}
Here we discuss the motivation behind our method and demonstrate our starting point on a very simple example. Consider the situation corresponding to the causal graph in Figure \ref{fig:motivatingexample}. In standard graphical representations, the nodes corresponding to noise are usually suppressed. Instead of the noise representation, we use the quantile representation. The quantile $U$ determines (together with the parent $A$) which value $X$ takes, meaning that
\begin{equation*}
X = g_X(A,U)
\end{equation*}
where $g_X$ is a deterministic function. It is important to note two facts:
\begin{itemize}
	\item there is a path $U \rightarrow X \rightarrow Y$ which implies that $U$ contains information on $Y$
	\item the path $U \rightarrow X \leftarrow A$ is blocked by $X$ and therefore $U \ci A$
\end{itemize}
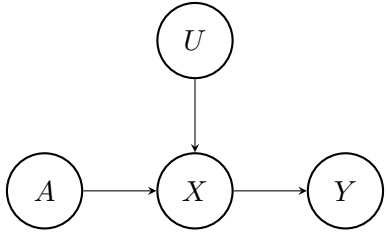
\begin{figure} \centering
	\begin{tikzpicture}
	[>=stealth, rv/.style={circle, draw, thick, minimum size= 10mm}, rvc/.style={triangle, draw, thick, minimum size=10mm}, node distance=25mm]
	\pgfsetarrows{latex-latex};
	\begin{scope}
	\node[rv] (1) at (-2,0) {$A$};
	\node[rv] (2) at (0,0) {$X$};
	\node[rv] (3) at (0,2) {$U$};
	\node[rv] (4) at (2,0) {$Y$};
	\draw[->] (1) -- (2);
	\draw[<-] (2) -- (3);
	\draw[->] (2) -- (4);
	\end{scope}
	\end{tikzpicture}
	\caption{A graphical model representation of a simple example motivating the main ideas in the section.}
	\label{fig:motivatingexample}
\end{figure}
We can immediately notice that this reasoning extends to a general causal graph. In any graph we have that the model distribution factorises as \citep{pearl2009}:
\begin{equation} \label{factorisation}
	f(x_1,...,x_k) = \prod_{i} f(x_i \mid pa(x_i))
\end{equation} The quantile $U_i$ of the corresponding distribution $f(x_i \mid pa(x_i))$ of feature $X_i$ contains information about $Y$ whenever $X_i$ does. This motivates the following proposition:
\begin{proposition} \label{mainprop}
	Consider a distribution $F_{(A,{X},Y)}$ of $(A,{X},Y)$ with a corresponding causal graph $\mathcal{G}$. Define the set $\mathbf{U} = \lbrace U_i : A \notin \de(X_i) \rbrace$, that is the set of quantiles corresponding to features that do not have $A$ as their descendant. Then the following hold:
	\begin{enumerate}[(i)]
		\item $\pa(\widehat{Y})\subset \mathbf{U} \implies \widehat{Y} \ci A$. Therefore, any such predictor $\widehat{Y}$ satisfies demographic parity. \label{firstcond}
		\item Under the (untestable) assumption that the quantiles in $\mathbf{U}$ remain unchanged under a $do(A = 0)$ intervention, our predictor $\widehat{Y}$ also satisfies individual level counterfactual fairness, in the sense that
		\begin{equation*}
			\pr(\widehat{Y} \mid {X} = {x}) = \pr(\widehat{Y}(A=0) \mid {X} = {x})
		\end{equation*}
		In words, this says that the distribution of the predictor does not change under the intervention on the protected attribute for any given individual. \label{untestable}
		\end{enumerate}
\end{proposition}
\begin{proof}
	For \ref{firstcond} notice that any path from $A$ to $U_i \in \mathbf{U}$ goes through $X_i$. Every path $A \leftrightarrow ... \rightarrow X_i \leftarrow U_i$ is blocked by the empty set $\emptyset$, since $X_i$ is a collider. It remains to check that there is no unblocked path $U_i \rightarrow X_i \rightarrow .... \leftrightarrow A$. Any such path either has a collider or $A \in \de(X_i)$ which is not the case. Therefore $A \ci U_i$, and since $\pa(\widehat{Y}) \subset \mathbf{U}$, it follows that $\widehat{Y} \ci A$. Claim \ref{untestable} is a direct consequence of the assumption that the quantiles $\mathbf{U}$ are unchanged under the $do(A = 0)$ intervention.
\end{proof}
We can now give the intuition for why quantiles are useful. In some sense, the space of quantiles offers a \textit{level playing field} where all levels of the protected attribute are treated the same. Using the quantiles allows us to map both levels $A = 0,1$ onto the same space and treat them equally. There is another, alternative way of doing this - by computing the actual counterfactual values ${X}(A=0)$  under a $do(A = 0)$ intervention. This would correspond to computing the value ${X}$  would have taken, had we hypothetically set $A=0$ for everyone. The latter approach is what we pursue in the paper.

\section{Proof of Proposition \ref{prop:NDE}} \label{appendix:NDE}
\begin{proof}
 Since $\widehat{Y}(A = a, {R} = {r}) \stackrel{d}{=}  \widehat{Y}(A = a', {R} = {r})\;\; \forall r$, then in particular it follows that $$\widehat{Y}(A = a, {R} = R(a')) \quad \stackrel{d}{=}\quad  \widehat{Y}(A = a', {R} = R(a')) .$$ Hence we have that \[E\big[  \widehat{Y}(A = a, {R} = R(a')) - \widehat{Y}(A=a')    \big] = E\big[  \widehat{Y}(A = a', {R} = R(a')) - \widehat{Y}(A=a')    \big] = 0  .\]
\end{proof}



\vskip 0.2in
\bibliographystyle{chicago}
\bibliography{fairnesslib}

\end{document}